\documentclass[10pt,twocolumn,letterpaper]{article}

\usepackage[pagenumbers]{wacv} % To force page numbers, e.g. for an arXiv version

\usepackage{graphicx}
\usepackage{amsmath}
\usepackage{amssymb}
\usepackage{booktabs}
\usepackage{csquotes}
\usepackage[table]{xcolor}
\usepackage{multirow}
\usepackage{adjustbox}
\usepackage{amsthm}

\usepackage[pagebackref,breaklinks,colorlinks]{hyperref}

\usepackage[capitalize]{cleveref}
\crefname{section}{Sec.}{Secs.}
\Crefname{section}{Section}{Sections}
\Crefname{table}{Table}{Tables}
\crefname{table}{Tab.}{Tabs.}

\def\wacvPaperID{*****} % *** Enter the WACV Paper ID here
\def\confName{WACV}
\def\confYear{2025}

\begin{document}

%%%%%%%%% TITLE - PLEASE UPDATE
\title{Test Time Adaptation Using Adaptive Quantile Recalibration}

% \begin{icmlauthorlist}
% \icmlauthor{Paria Mehrbod}{concordia,mila}
% \icmlauthor{Pedro Vianna}{udem,mila}
% \icmlauthor{Geraldin Nanfack}{concordia,mila}
% \icmlauthor{Guy Wolf}{udem,mila}
% \icmlauthor{Eugene Belilovsky}{concordia,mila}
% \end{icmlauthorlist}

% \icmlaffiliation{udem}{Université de Montréal}
% \icmlaffiliation{mila}{Mila -- Quebec AI Institute}
% \icmlaffiliation{concordia}{Concordia University}

% \icmlcorrespondingauthor{Paria Mehrbod}{paria.mehrbod@mila.quebec}
% \author{Paria Mehrbod\\
% Institution1\\
% Institution1 address\\
% {\tt\small firstauthor@i1.org}
% % For a paper whose authors are all at the same institution,
% % omit the following lines up until the closing ``}''.
% % Additional authors and addresses can be added with ``\and'',
% % just like the second author.
% % To save space, use either the email address or home page, not both
% \and
% Second Author\\
% Institution2\\
% First line of institution2 address\\
% {\tt\small secondauthor@i2.org}
% }

\author{
Paria Mehrbod$^{1,2}$ \quad
Pedro Vianna$^{3,2}$ \quad
Geraldin Nanfack$^{1,2}$ \quad
Guy Wolf$^{3,2}$ \quad
Eugene Belilovsky$^{1,2}$\\[1em]
$^1$Concordia University \quad
$^2$Mila -- Quebec AI Institute \quad
$^3$Université de Montréal\\
% Montreal, Quebec, Canada\\[0.5em]
% {\tt\small paria.mehrbod@mila.quebec}
}
\maketitle
\newcolumntype{L}{>{\columncolor{blue!10}}c}
\newcommand{\RC}[1]{\cellcolor{red!10}{#1}}

%%%%%%%%% ABSTRACT
\begin{abstract}
% old abstract
% Domain adaptation methods have emerged as effective mechanisms to improve the generalizability and robustness of deep learning models, particularly in real-world scenarios where test data may differ significantly from the training domain. However, traditional domain adaptation techniques often require prior knowledge of target domains or model retraining, which limits their applicability in dynamic settings where such information is unavailable or retraining is impractical. 
% Approaches based on updating batch normalization statistics at test-time have been gaining traction, as they allow for unsupervised adaptation based on the target data. Some of these approaches only adjust batch normalization statistics and do not fully capture complex distributions, and are restricted to specific normalization types. To address this, we propose Adaptive Quantile Recalibration (AQR), a novel test-time adaptation method based on quantile recalibration, which modifies the pre-activation distributions by aligning quantiles on a channel-by-channel basis.
% AQR captures the complete shape of activation distributions and works across diverse architectures regardless of normalization type (BatchNorm, GroupNorm, or LayerNorm). We demonstrate that our method provides robust adaptation across diverse settings, outperforming baseline test-time adaptation methods.
% Suggestion:
Domain adaptation is a key strategy for enhancing the generalizability of deep learning models in real-world scenarios, where test distributions often diverge significantly from the training domain. However, conventional approaches typically rely on prior knowledge of the target domain or require model retraining, limiting their practicality in dynamic or resource-constrained environments. Recent test-time adaptation methods based on batch normalization statistic updates allow for unsupervised adaptation, but they often fail to capture complex activation distributions and are constrained to specific normalization layers. We propose Adaptive Quantile Recalibration (AQR), a test-time adaptation technique that modifies pre-activation distributions by aligning quantiles on a channel-wise basis. AQR captures the full shape of activation distributions and generalizes across architectures employing BatchNorm, GroupNorm, or LayerNorm. To address the challenge of estimating distribution tails under varying batch sizes, AQR incorporates a robust tail calibration strategy that improves stability and precision. Our method leverages source-domain statistics computed at training time, enabling unsupervised adaptation without retraining models. Experiments on CIFAR-10-C, CIFAR-100-C, and ImageNet-C across multiple architectures demonstrate that AQR achieves robust adaptation across diverse settings, outperforming existing test-time adaptation baselines. These results highlight AQR’s potential for deployment in real-world scenarios with dynamic and unpredictable data distributions.
\end{abstract}

\section{Introduction}
\label{sec:intro}

Deep neural networks have demonstrated remarkable success across numerous computer vision tasks, including image classification, object detection, and segmentation. However, these models often suffer significant performance degradation when deployed in real-world environments that differ from their training conditions. This phenomenon, known as distribution shift or domain gap, poses a major challenge for the practical deployment of deep learning systems in applications where reliability and robustness are critical. 

% Test-time adaptation (TTA) has emerged as a promising approach to address this challenge by adapting models to target distributions during inference without requiring access to the original training data or modifying the training process. Unlike traditional domain adaptation methods that assume prior knowledge of the target domain or require retraining, TTA methods operate solely on test data batches, making them particularly suitable for real-world applications where distribution shifts may occur unexpectedly or evolve over time.
Several domain adaptation methods have been proposed to address this issue, although they often assume prior knowledge of the target domain or require retraining, which hinders their applicability across different tasks and scenarios. Test-time adaptation (TTA) techniques have emerged as promising approaches that adapt models to target distributions during inference, relying solely on test data batches. These techniques are particularly suitable for real-world applications where distribution shifts may occur unexpectedly or evolve over time\cite{liang2025comprehensive}.

A popular approach to TTA is based on test-time normalization (TTN), where batch normalization statistics are modified to match the target data distribution. This method has been demonstrated to be particularly effective for convolutional neural networks (CNNs) \cite{nado2020evaluating, schneider2020improving,vianna2024hybridttn}, and recently has achieved notable success when applied to vision transformers (ViTs) \cite{marsden2023roid, lee2024cmf, lee2024deyo, wang2024otta}. However, TTN implicitly assumes the neuron-level activations approximate a Gaussian distribution, which may not hold for complex, multi-modal distributions encountered in practice. Furthermore, TTN is limited to architectures that employ batch normalization layers (BatchNorm\cite{batchnorm}), making it inapplicable to models using other normalization schemes such as group normalization or layer normalization (GroupNorm\cite{groupnorm} and LayerNorm\cite{layernorm}).
\begin{figure}
    \centering    \includegraphics[width=\linewidth]{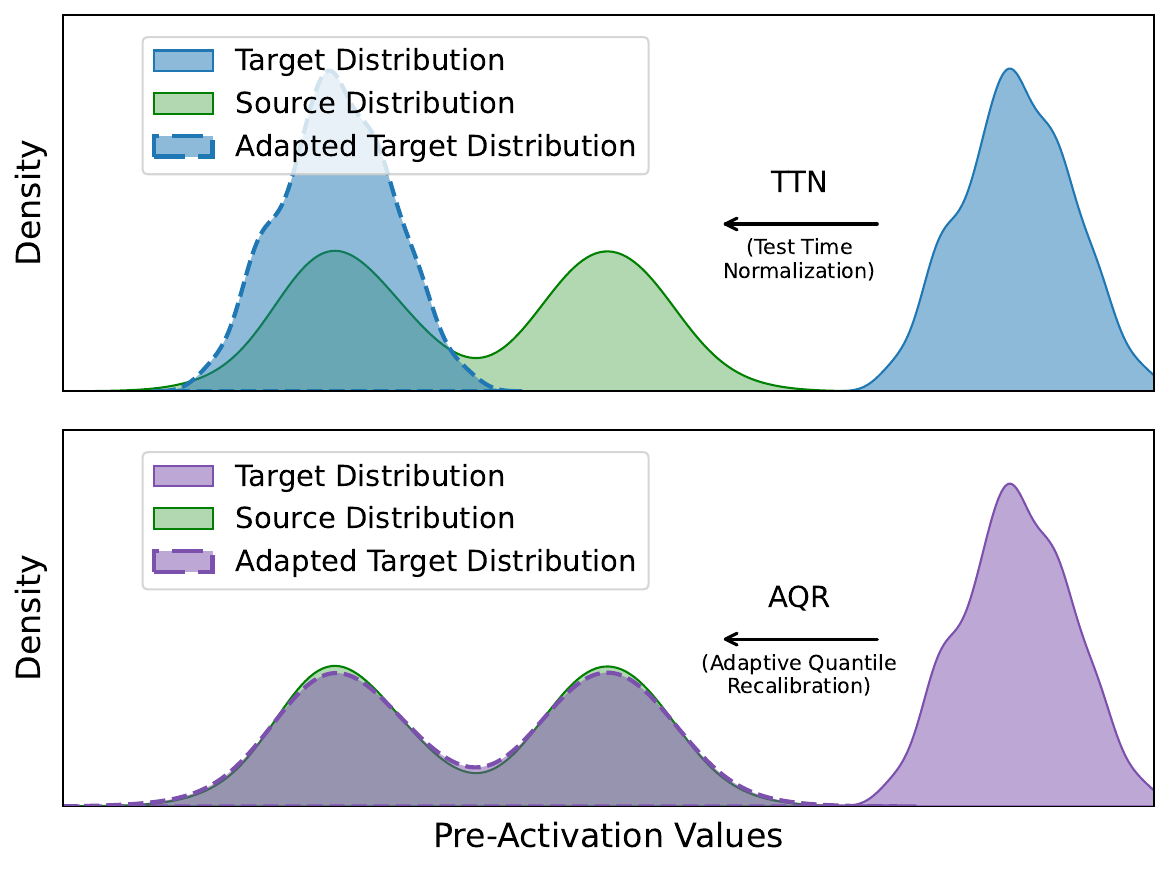}
    \vspace{-9mm}
    \caption{Comparing AQR and TTN in preserving complex distribution shapes at test-time using synthetic data.}
    \vspace{-5mm}
    \label{fig:kde_plots}
\end{figure}

%Comment: Specifically for Tent, Park 2023 says that Tent (or entropy minimization in general) are biased towards majority classes from the training data. I don't know if you would like to add this criticism here.
% Sunghyun Park, Seunghan Yang, Jaegul Choo, and Sungrack Yun. Label shift adapter for test-time adaptation under covariate and label shifts. In Proceedings of the IEEE/CVF International Conference on Computer Vision, pp. 16421–16431, 2023.

We thus propose Adaptive Quantile Recalibration (AQR), a novel TTA method that aligns the distributions of internal features between source and target domains, without relying on parametric distribution assumptions. Our approach leverages nonparametric quantile-based transformations to map target domain activations to their corresponding source domain distributions on a channel-by-channel basis. Unlike methods that only adjust the mean and variance of pre-activations, AQR captures and preserves the complete shape of pre-activation distributions, making it effective for handling complex distribution patterns commonly found in deep neural networks (Figure \ref{fig:kde_plots}).
A critical advantage of our method is that it does not degrade over time, as we are adapting to source activations that are precomputed and remain fixed throughout testing. Unlike entropy-based methods that continuously update model parameters and can drift toward suboptimal solutions, AQR provides a stable reference point derived from the source domain statistics, ensuring consistent and stable adaptation performance even in challenging test scenarios.
Our key contributions are as follows:
\begin{itemize}
\vspace{-2mm}
\item We propose a novel method that calibrates pre-activations at test-time to align with train-time pre-activations by leveraging statistics computed at the end of model training.
\item  We demonstrate our method's applicability across diverse model architectures, independent of specific types of normalization layers.
\item We identify and address challenges associated with varying batch sizes in computing statistical information and propose strategies for the accurate estimation of distribution tails.
\item  Our experiments on three datasets across four architectures show that our method outperforms current state-of-the-art approaches and shows potential for real-world applications.

\end{itemize}

%------------------------------------------------------------------------

\section{Related Work}
\label{sec:related}
% TTN \cite{schneider2020improving,nado2020evaluating}
% not adaptable to other architectures that include batch norm
TTA methods frequently operate by updating the parameters associated with BatchNorm layers in response to covariate shifts in the input distribution\cite{TENT,yuan2023robust,nado2020evaluating,schneider2020improving,wu2024test,lessBN,lim2023ttn}, as is the case of the popular approach TTN \cite{schneider2020improving,nado2020evaluating}. This strategy is closely related to the unsupervised domain adaptation technique AdaBN \cite{li2016revisiting}, and 
has demonstrated effectiveness in mitigating the effects of varying degrees of image corruption. However, a key limitation lies in their dependency on BatchNorm, which restricts their applicability to architectures using this specific normalization scheme. Additionally, BatchNorm is considered a major contributor to instability in standard TTA pipelines \cite{niu_towards_2023}.

The rising use of alternative normalization layers like GroupNorm and LayerNorm necessitates the development of TTA algorithms compatible with various architectures.
% SAR \cite{niu_towards_2023}
% uses relibale samples to update the affine parameters but it could degrade overtime (our method does not degrade overtime, as it does not change any model parameters)
% While batch-agnostic normalization techniques could generally provide more stable performance, they are still susceptible to failure under domain shift. 
Addressing these challenges, sharpness-aware and reliable entropy minimization (SAR) \cite{niu_towards_2023} was developed as an online TTA method that supports all types of normalization layers.

Another popular approach, often combined with TTN, requires adapting the affine parameters of normalization layers using entropy minimization loss, as seen in %methods like \cite{TENT} and \cite{niu_towards_2023}. 
previous research \cite{TENT, niu_towards_2023}.
However, these methods face stability challenges in wild test scenarios. Specifically, they can produce faulty feedback if an incorrect selection of samples is used to calculate the loss and may suffer performance degradation over time.
% todo: move tent here
% SAR selectively updates model parameters based on reliable test samples. However, SAR relies on updating affine parameters 
% \cite{niu_towards_2023}

For easier deployment across multiple architectures, marginal entropy minimization with one test point (MEMO) \cite{zhang_memo_2022} was proposed as an approach needing only the trained model and a single test input. However, it is computationally expensive due to per-sample backpropagation and test-time augmentation, making it unsuitable for latency-sensitive tasks.

Neuron editing \cite{amodio_neuron_editting} addresses the problem of generating transformed versions of data based on the pre- and post-transformation versions observed of a small subset of the available data. Rather than learning distribution-to-distribution mappings, it reframes the problem as learning a general edit function that can be applied to other datasets. The method applies piecewise linear transformations to neuron activations in autoencoder latent spaces, computing percentile-based differences between source and target distributions. 
This nonparametric approach preserves data variability and avoids issues like mode collapse seen in generative models. Inspired by this approach, we adapt their percentile-based transformation strategy to the test-time adaptation setting while making it applicable to diverse neural network architectures independent of specific normalization layers.

\section{Methodology}
\begin{figure*}[!htbp]
    \centering
    \includegraphics[width=\textwidth]{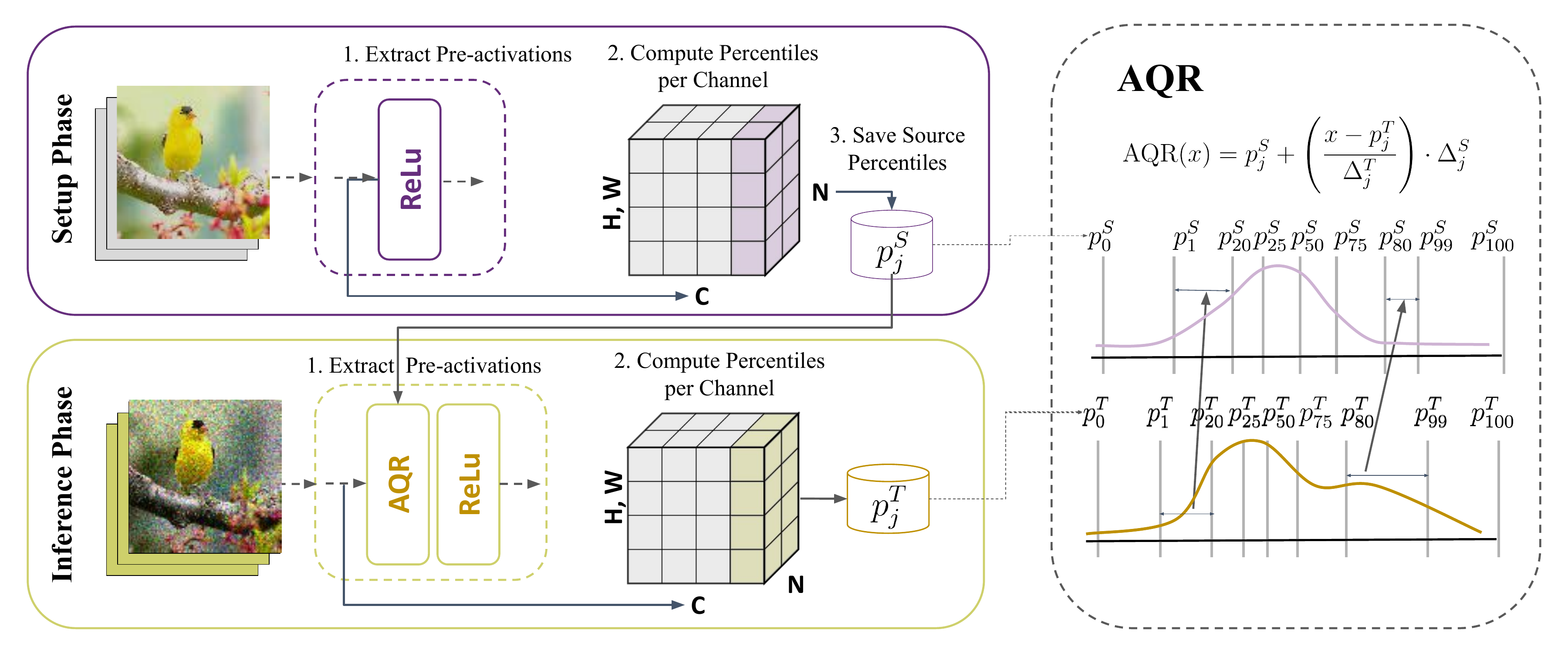}
    \caption{Overview of the Adaptive Quantile Recalibration (AQR) method. During the setup phase (top), source data is processed to extract pre-activations and compute percentiles per channel, which are saved as reference statistics. During inference (bottom), target data pre-activations are similarly processed, and AQR transforms target percentiles to match source percentiles using piecewise linear transformation, enabling distribution alignment without architectural constraints.}
    \label{fig:your-label}
\end{figure*}

In the current work, we propose a TTA method based on aligning the distributions of the intermediate features of a neural network. Our key insight is that distribution shifts between training and testing data manifest as shifts in the distributions of intermediate features of neural networks. By transforming these internal distributions to match those observed during training, we can improve the model's performance on out-of-distribution test data without requiring access to training data or modifying the training process.
Our method consists of two phases: First, in the setup phase, we compute the statistical information of the internal layers of the model when given the source data. Second, in the inference phase, we transform the model's internal pre-activation values to correct for distribution shifts that occur when processing test data.

\subsection{Setup Phase: Source Distribution Statistics}
Let $f_\theta$ denote a neural network with parameters $\theta$ trained on source distribution $P(x)$. After training is complete and before any inference, we perform a one-time setup phase to capture the statistical information of the source distribution. In this phase, we apply the following steps:

1) Process a subset of source/training data $S$ through the trained model.

2) For each layer $l$, and each channel $c$ within that layer, store the pre-activation values denoted as $a_c^l$ (outputs of normalization layers before activation function).

3) Compute percentiles $p_i^S$ where $i \in \{0,1,\dots,100\}$ from the stored pre-activation values $a_c^l$, doing this separately for each channel in each layer. 

These stored percentiles ($p_i^S$) serve as a memory of the distribution characteristics of the model's internal values when processing in-distribution data, and will be used during inference to guide the adaptation process. 

\subsection{Inference Phase: Distribution Alignment}
During inference, when out-of-distribution test samples are processed through the network, the distribution of pre-activation values ($a_c^l$) deviates from what was observed during training. We propose to transform these values to match their training-time distributions.

Our method is agnostic to the specific type of normalization layer used in the network (batch, layer, or group normalization). For each batch of test samples, we:
1) compute percentiles $p_i^T$ of the pre-activation values for each channel,
2) transform these values using a piecewise linear transformation adapted from \cite{amodio_neuron_editting} that we denote \textbf{AQR}. This transformation is applied as follows:
% \begin{equation}
% \resizebox{\columnwidth}{!}{$\text{AQR}(x) = 
% \begin{cases} 
% \left( \dfrac{a - p_0^T}{p_1^T - p_0^T} \cdot \left( p_1^S - p_0^S \right) \right) + p_0^S & a < p_1^T, \\
% \left( \dfrac{a - p_j^T}{p_{j+1}^T - p_j^T} \cdot \left( p_{j+1}^S - p_j^S \right) \right) + p_j^S & a \in [p_j^T, p_{j+1}^T], \\
% \left( \dfrac{a - p_{99}^T}{p_{100}^T - p_{99}^T} \cdot \left( p_{100}^S - p_{99}^S \right) \right) + p_{99}^S & a \geq p_{99}^T,
% \end{cases}
% $}
% \end{equation}
\begin{equation}
\text{AQR}(x) = p_j^S + \left( \dfrac{x - p_j^T}{\Delta_j^T} \right) \cdot \Delta_j^S \quad \text{for } x \in [p_j^T, p_{j+1}^T)
\label{eq:aqr}
\end{equation}

\noindent where, $\Delta_j^T = p_{j+1}^T - p_j^T$ and $x$ represents the pre-activation values of a specific channel and a specific layer, $p_i^T$ represents the $i$-th percentile of the test samples' pre-activation values (computed on-the-fly), and $p_i^S$ represents the $i$-th percentile of the source/training pre-activation values (previously computed during the setup phase). This transformation uses 100 percentile intervals, with $j \in \{0,1,2,...,99\}$ covering the entire distribution range from the 0th to the 100th percentile, and maps the test-time distribution back to the distribution observed during training. This transformation is applied to all channels of a given model.

% \subsection{Channel-wise Adaptation Process}
% Instead of operating on individual neurons, our method performs transformations at the channel level. This design choice is motivated by the observation that channels in neural networks often capture coherent features \cite{vianna2024hybridttn}, making channel-wise distribution adjustments more meaningful than neuron-wise adjustments. 

% % this could be a pseudo-code:
% The complete adaptation process involves:
% %I would change the sentence above to match the previous subsection.
% % In this phase, we apply the following steps:
% % And than use 1) and 2) like before, instead of itemize. Or itemize above. My point is: do the same thing for both sections.

% \textbf{Setup Phase (One-time):}
% \begin{itemize}
%     \item Select a subset of training data and process them using the trained model
%     \item Compute and store percentiles pSip_i^S for each channel from the normalized values
% \end{itemize}

% \textbf{Inference Phase (Per-batch):}
% \begin{itemize}
%     \item Compute percentiles pTip_i^T for each channel from current test samples
%     \item Apply the AQR transformation after each normalization layer using the pre-computed pSip_i^S and current pTip_i^T
% \end{itemize}

% % Or, if you want to make it a pseudo-code, I would suggest doing it here and removing the steps 1,2,3 from the Setup subsection.

\subsection{Calibrating the Tail of Distribution} %Or Calibrating Distribution Tails?
The size of the source dataset can affect how accurately the source distribution is estimated. More samples lead to better overall estimation, but can produce extreme values in the distribution tails. Figure \ref{fig:activation_diff_boxplot} demonstrates the instability of tail percentile estimation using small batches. We drew 20 different batches of 128 samples from the source/training distribution and computed percentiles for each batch. For each percentile level, we calculated the deviation from the source/training percentile computed using 10,000 training samples. Each boxplot shows the distribution of these deviations across the 20 batches. The results reveal that tail percentiles show substantial variability: the 0th percentile (minimum) consistently overestimates the true minimum, while the 100th percentile (maximum) consistently underestimates the true maximum. This bias and high variability in tail estimation motivate our tail calibration strategy. Instead of using actual minimum and maximum values of the source data, we estimate the first and last percentiles through sampling. We compute these statistics over a batch of 100, repeat the sampling 1,000 times, and then average the results. This approach provides more reliable estimates of the distribution tails. We evaluate the impact of this strategy in the following section.
\begin{figure}
    \centering
    \includegraphics[width=\linewidth]{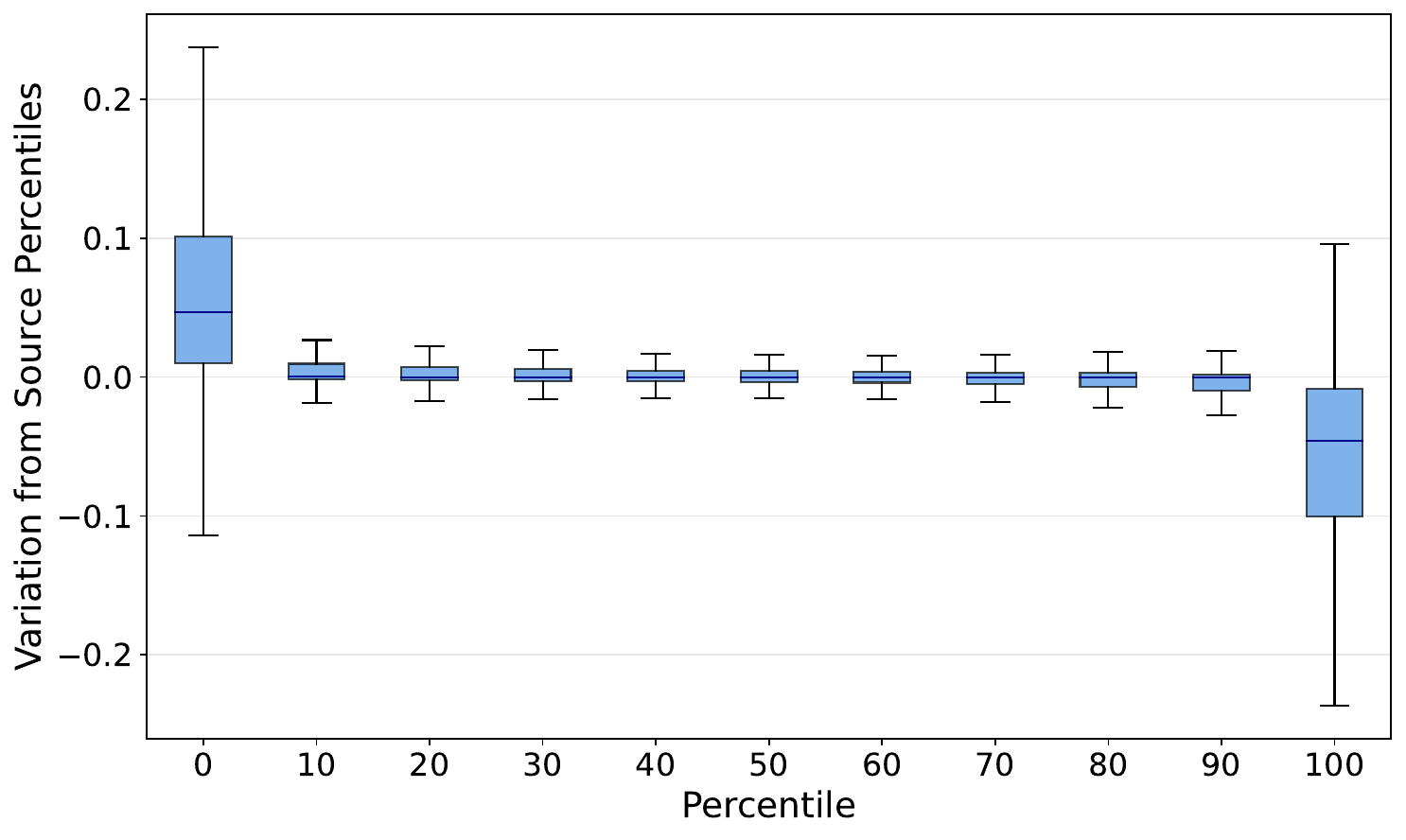}
        \vspace{-9mm}
    \caption{Distribution of deviations between small-batch (128) and reference (10,000) percentiles across 20 trials.}
        \vspace{-6mm}
    \label{fig:activation_diff_boxplot}
\end{figure}
\begin{table*}[htbp]
\centering
\caption{Comparing the performance of test-time adaptation methods on ImageNet-C at corruption severity level 3 and batch size 128. Results show classification accuracy (\%) per corruption type. Bold indicates best performance per model. "NA" denotes the non-adapted (original) model. Standard Error does not exceed 0.27.}
\label{tab:imagenet_compact}
\renewcommand{\arraystretch}{1.3}
\resizebox{\textwidth}{!}{
\LARGE
\begin{tabular}{m{1cm}l|cccc|cccc|cccc|ccccc|cc|L}
\toprule
\multirow{2}{*}{\rotatebox{90}{\textbf{Model}}} & \textbf{Method} & \multicolumn{4}{c|}{\textbf{Noise}} & \multicolumn{4}{c|}{\textbf{Blur}} & \multicolumn{4}{c|}{\textbf{Weather}} & \multicolumn{5}{c|}{\textbf{Distortion}} & \multicolumn{2}{c|}{\textbf{Digital}} & \textbf{Average} \\
 &  & \textbf{Gauss.} & \textbf{Impul.} & \textbf{Shot} & \textbf{Speck.} & \textbf{Defoc.} & \textbf{G.Blur} & \textbf{Motion} & \textbf{Zoom} & \textbf{Bright.} & \textbf{Contr.} & \textbf{Satur.} & \textbf{Fog} & \textbf{Elastic} & \textbf{Frost} & \textbf{Glass} & \textbf{Pixel} & \textbf{Snow} & \textbf{JPEG} & \textbf{Spatt.} & \textbf{Avg} \\
\midrule
\vspace{2mm}
\multirow{5}{*}{\rotatebox{90}{ResNet50 (BN)}} & NA & 27.6 & 25.1 & 25.1 & 31.7 & \textbf{38.0} & 41.6 & 37.7 & 35.2 & 69.6 & 46.0 & 71.4 & 46.6 & 55.6 & 32.1 & 16.9 & 46.2 & 35.2 & 59.3 & 49.4 & 41.6 \\
 & TTN & 45.5 & 44.5 & 43.2 & 47.4 & 37.2 & 42.3 & 50.2 & 51.2 & 72.1 & 64.1 & 73.8 & 62.2 & 66.4 & 41.6 & 32.4 & 64.4 & 47.5 & 63.0 & 59.4 & 53.1 \\
 & TENT & 45.5 & 44.5 & 43.2 & 47.4 & 37.2 & 42.3 & 50.3 & 51.2 & 72.1 & 64.1 & 73.8 & 62.3 & 66.4 & 41.6 & 32.4 & 64.4 & 47.5 & 63.0 & 59.4 & 53.1 \\
 & SAR & 45.6 & 44.6 & 43.3 & 47.5 & 37.4 & \textbf{42.5} & 50.4 & 51.3 & 72.1 & 64.1 & \textbf{73.8} & 62.3 & 66.5 & 41.7 & 32.6 & 64.5 & 47.6 & 63.0 & 59.5 & 53.2 \\
 & AQR & \RC{\textbf{48.3}} & \RC{\textbf{47.2}} & \RC{\textbf{47.4}} & \RC{\textbf{50.2}} & \RC{35.1} & \RC{39.6} & \RC{\textbf{51.3}} & \RC{\textbf{52.1}} & \RC{\textbf{72.1}} & \RC{\textbf{64.8}} & \RC{73.3} & \RC{\textbf{63.1}} & \RC{\textbf{68.0}} & \RC{\textbf{45.5}} & \RC{\textbf{33.8}} & \RC{\textbf{65.0}} & \RC{\textbf{51.0}} & \RC{\textbf{64.6}} & \RC{\textbf{61.5}} & \RC{\textbf{54.4}} \\
\midrule
\vspace{2mm}
\multirow{4}{*}{\rotatebox{90}{ResNet50 (GN)}} & NA & 54.5 & 53.1 & 52.8 & 57.7 & 44.3 & 49.8 & 49.7 & 39.2 & 75.4 & 69.8 & 76.9 & 55.8 & 59.6 & \textbf{54.0} & 21.2 & 59.7 & 54.8 & 66.3 & 63.3 & 55.7 \\
 & TENT & 54.5 & \textbf{53.1} & 52.8 & 57.7 & 44.3 & 49.8 & 49.7 & 39.2 & 75.4 & 69.8 & 76.9 & 55.8 & 59.6 & 54.0 & 21.2 & 59.7 & 54.8 & 66.3 & 63.3 & 55.7 \\
 & SAR & \textbf{54.5} & 53.1 & \textbf{52.9} & \textbf{57.7} & \textbf{44.3} & \textbf{49.8} & 49.7 & 39.2 & \textbf{75.4} & 69.8 & \textbf{76.9} & 56.1 & 59.6 & 54.0 & 21.3 & 59.8 & 54.8 & \textbf{66.3} & 63.2 & 55.7 \\
 & AQR & \RC{50.9} & \RC{48.0} & \RC{48.1} & \RC{50.1} & \RC{41.7} & \RC{45.4} & \RC{\textbf{59.2}} & \RC{\textbf{55.5}} & \RC{74.9} & \RC{\textbf{71.3}} & \RC{76.1} & \RC{\textbf{68.7}} & \RC{\textbf{71.1}} & \RC{51.4} & \RC{\textbf{36.9}} & \RC{\textbf{67.9}} & \RC{\textbf{57.1}} & \RC{64.9} & \RC{\textbf{65.0}} & \RC{\textbf{58.1}} \\
\midrule
\vspace{2mm}
\multirow{4}{*}{\rotatebox{90}{ViT-Base (FT)}} & NA & 62.7 & 62.0 & 60.9 & 63.3 & 51.9 & 54.6 & 58.2 & 45.0 & 72.6 & 76.8 & 75.4 & 71.0 & 67.9 & 34.4 & 37.1 & 69.2 & 45.3 & 67.8 & 64.1 & 60.0 \\
 & TENT & 53.8 & 53.1 & 50.1 & 53.8 & 48.6 & 51.6 & 53.3 & 40.8 & 68.4 & 74.3 & 71.0 & 67.7 & 64.9 & 31.6 & 35.4 & 65.4 & 35.8 & 64.7 & 59.5 & 54.9 \\
 & SAR & 62.7 & 58.5 & 60.3 & 63.3 & 52.0 & 54.7 & 58.3 & 45.1 & 72.6 & 76.8 & 75.4 & 71.0 & 67.3 & 34.5 & 37.2 & 69.0 & 45.6 & 67.7 & 64.1 & 59.8 \\
 & AQR & \RC{\textbf{63.9}} & \RC{\textbf{63.3}} & \RC{\textbf{62.3}} & \RC{\textbf{64.9}} & \RC{\textbf{55.2}} & \RC{\textbf{57.8}} & \RC{\textbf{59.7}} & \RC{\textbf{52.2}} & \RC{\textbf{75.1}} & \RC{\textbf{77.7}} & \RC{\textbf{77.5}} & \RC{\textbf{72.9}} & \RC{\textbf{73.5}} & \RC{\textbf{41.8}} & \RC{\textbf{48.1}} & \RC{\textbf{72.6}} & \RC{\textbf{52.6}} & \RC{\textbf{71.7}} & \RC{\textbf{69.7}} & \RC{\textbf{63.8}} \\
\midrule
\vspace{2mm}
\multirow{4}{*}{\rotatebox{90}{ViT-Base (TS)}} & NA & 39.3 & 37.9 & 36.5 & 43.8 & 42.5 & 46.1 & 48.8 & 40.8 & 64.9 & 66.7 & 67.9 & 56.5 & 67.9 & 30.2 & 37.6 & 68.0 & 31.5 & 64.6 & 53.8 & 49.7 \\
 & TENT & 39.3 & 37.9 & 36.5 & 43.7 & 42.5 & 46.1 & 48.8 & 40.8 & 64.9 & 66.7 & 67.9 & 56.5 & 67.9 & 30.2 & 37.6 & 68.0 & 31.5 & 64.6 & 53.8 & 49.7 \\
 & SAR & 39.4 & 37.9 & 36.5 & 43.8 & 42.5 & 46.1 & 48.8 & 40.9 & 64.9 & 66.7 & 67.9 & 56.5 & 67.9 & 30.3 & 37.6 & 68.0 & 31.6 & 64.6 & 53.8 & 49.8 \\
 & AQR & \RC{\textbf{43.8}} & \RC{\textbf{43.2}} & \RC{\textbf{41.6}} & \RC{\textbf{47.7}} & \RC{\textbf{44.9}} & \RC{\textbf{48.3}} & \RC{\textbf{51.1}} & \RC{\textbf{44.7}} & \RC{\textbf{65.6}} & \RC{\textbf{67.9}} & \RC{\textbf{68.3}} & \RC{\textbf{58.5}} & \RC{\textbf{68.9}} & \RC{\textbf{35.3}} & \RC{\textbf{42.4}} & \RC{\textbf{68.7}} & \RC{\textbf{37.7}} & \RC{\textbf{65.4}} & \RC{\textbf{57.3}} & \RC{\textbf{52.7}} \\

\bottomrule
\end{tabular}
}
\end{table*}

\subsection{Analysis of a One-Hidden-Layer Model}
\label{sec:anlysis_one_hidden_layer}

In this section, we provide a simple theoretical motivation under a simplified architecture and corruption model for advantages of AQR's piecewise-linear transformation. This proof holds under specific assumptions and is intended to provide intuition for how AQR works compared to TTN and similar methods that only update affine parameters. In App.~\ref{app:finite-sample-aqr}, we present a complementary analysis under finite samples and quantile discretization.
\paragraph{Main objective.}
We compare two adaptation strategies, AQR and TTN, which attempt to recover the source hidden representation \(h^S\) from the corrupted test-time \(h^T\). We measure adaptation quality using the total mean squared error (MSE)
\[
\mathrm{MSE}(T) \ := \ \sum_{i=1}^m \mathbb{E}\big[(T_i(h_i^T) - h_i^S)^2\big].
\]
Our goal is to show that under the assumptions below,
\[
\mathrm{MSE}\!\big(T^{\mathrm{AQR}}\big) \ = \ 0,
\
\mathrm{MSE}\!\big(T^{\mathrm{TTN}}\big) \ > \ 0
\ \text{for non-affine} \ k_i .
\]

\paragraph{Setting.}
We consider a one-hidden-layer MLP on the \textbf{source domain}. Let the input be a random vector \(x\in\mathbb{R}^d\) with
\[
x \sim P \ \text{ on } \ \mathbb{R}^d, \qquad \mathrm{supp}(P)=\mathbb{R}^d.
\]
Network parameters are \( W \in \mathbb{R}^{m \times d} \) and \( b \in \mathbb{R}^m \). The pre-activations and post-activations are
\[
a^S = Wx + b \in \mathbb{R}^m, \qquad h^S = \phi(a^S) \in \mathbb{R}^m,
\]
with final prediction \( y_S = w^\top h^S \) for some \( w \in \mathbb{R}^m \).
The activation \(\phi:\mathbb{R}\to\mathbb{R}\) is strictly increasing and continuous (e.g., identity or leaky-ReLU with positive slope), so \(\phi^{-1}\) is well-defined on its image.
Index neurons by \(i\in\{1,\dots,m\}\). For each \(i\), the source hidden \(h_i^S\) has marginal distribution \(P_i\) with continuous density and strictly increasing CDF \(F_{P_i}\).
\paragraph{Corruption model.}
Test-time corruptions are modeled as strictly increasing transformations that, although induced by the input shift, are observed at pre-activations/activations. Concretely, for each \(i\) there exists strictly increasing \(g_i:\mathbb{R}\to\mathbb{R}\) such that
\begin{align*}
a_i^T \ \stackrel{d}{=} \ g_i(a_i^S).
% \label{eq:corr-on-x}
\end{align*}
Since \(\phi\) is strictly increasing, we define
\begin{align*}
k_i \ := \ \phi \circ g_i \circ \phi^{-1},
% \label{eq:k_i}
\end{align*}
which is strictly increasing. Then the corrupted hidden representation satisfies
\begin{align}
h_i^T \ = \ k_i(h_i^S),
\label{eq:corruption_h}
\end{align}
and has target marginal \( Q_i = (k_i)_{\#} P_i \) (pushforward of \(P_i\) through \(k_i\)). If \(F_{Q_i}\) is the CDF of \(Q_i\), then for all \(x\in\mathrm{supp}(Q_i)\),
% CDF pushforward (number only if you cite it)
\begin{equation*}
F_{Q_i}(x) = F_{P_i}\!\big(k_i^{-1}(x)\big).
\label{eq:cdf-pushforward}
\end{equation*}

\paragraph{Comparison of methods.}
\textbf{AQR} uses the exact CDFs to define the quantile transform
\begin{equation*}
T_i^{\mathrm{AQR}}(z) := F_{P_i}^{-1}\!\big(F_{Q_i}(z)\big).
\label{eq:aqr-map}
\end{equation*}
Applying this to \(h_i^T\) gives
\begin{align*}
T_i^{\text{AQR}}(h_i^T)
&= F_{P_i}^{-1}\!\big(F_{Q_i}(h_i^T)\big) \\
&= F_{P_i}^{-1}\!\big(F_{P_i}(k_i^{-1}(h_i^T))\big) \\
&= k_i^{-1}(h_i^T) \\
&= k_i^{-1}\!\big(k_i(h_i^S)\big) \quad \text{by Eq.~\eqref{eq:corruption_h}} \\
&= h_i^S.
\end{align*}

\noindent
\textbf{TTN} applies the affine transform
\[
T_i^{\mathrm{TTN}}(z) \ = \ \mu_i^S \ + \ \sigma_i^{S}\,\frac{z-\mu_i^T}{\sigma_i^T},
\]
where \((\mu_i^S,\sigma_i^S)\) and \((\mu_i^T,\sigma_i^T)\) are the (population) mean and standard deviation of \(P_i\) and \(Q_i\), respectively. This matches first and second moments, and exactly inverts the corruption only when \(k_i\) is affine; for nonlinear \(k_i\), TTN leaves a nonzero distortion.

\paragraph{Result.}
From the AQR derivation above,
\[
\mathrm{MSE}\!\big(T^{\mathrm{AQR}}\big)
= \sum_{i=1}^m \mathbb{E}\big[(T_i^{\mathrm{AQR}}(h_i^T) - h_i^S)^2\big] \ = \ 0.
\]
In contrast,
\[
\mathrm{MSE}\!\big(T^{\mathrm{TTN}}\big) \ > \ 0
\quad \text{whenever some } k_i \text{ is non-affine}.
\]
Therefore, AQR achieves perfect recovery under these idealized conditions, whereas TTN cannot unless each \(k_i\) is affine. A complementary finite-sample/discretization analysis is provided in App.~\ref{app:finite-sample-aqr}.

\paragraph{Remark (MSE aggregation).}
We sum MSE across neurons; equivalently, one may average by \(m\) or write \(\mathbb{E}\!\left[\|T(h^T)-h^S\|_2^2\right]\). The conclusions are unchanged.

\section{Experiments and Results}

\subsection{Datasets and Models}

We evaluate our methods on CIFAR-10, CIFAR-100, and ImageNet-1K together with their corresponding corruption benchmarks, CIFAR-10-C, CIFAR-100-C, and ImageNet-C \cite{krizhevsky2009learning,hendrycks2019benchmarking}. Each corruption suite contains 19 corruption types, and each type is provided at five severity levels. We report results at severity levels 1, 3, and 5. CIFAR-10 and CIFAR-100 each contain 50{,}000 training images and 10{,}000 test images at a resolution of $32 \times 32$, with 10 and 100 classes respectively. ImageNet-1K contains 1{,}280{,}000 training images and 50{,}000 validation images at a resolution of $224 \times 224$ across 1{,}000 classes.
% Experiments: Models and Training Regimes (concise)

We evaluate four architecture families per dataset: ResNets with BatchNormalization (BN) or GroupNormalization (GN), and Vision Transformers (ViT) with LayerNormalization (LN). For ViTs, we consider two regimes: trained from scratch (TS) on the target dataset and fine-tuned (FT) after pre-training on a larger corpus. All fine-tuned ViTs are pre-trained on ImageNet-21K and then fine-tuned on the dataset at hand. We denote ViT-Base-patch16-224 as ViT-Base and ViT-patch4-32 as ViT-Small.
Below we list, for each dataset, the specific models and the ViT training regime (TS/FT).
\begin{itemize}
  \item \textbf{CIFAR-10:} ResNet-18 (BN); ResNet-26 (GN); ViT-Small (TS); ViT-Base (FT).
  \item \textbf{CIFAR-100:} ResNet-50 (BN); ResNet-50 (GN); ViT-Small (TS); ViT-Base (FT).
  \item \textbf{ImageNet-1K:} ResNet-50 (BN); ResNet-50 (GN); ViT-Base (trained from scratch with Sharpness Aware Minimization); ViT-Base (FT).
\end{itemize}
Each set of experiments is evaluated over three random seeds to ensure statistical significance.
% For each dataset, we test four model architectures: ResNet variants with BatchNorm (BN) and GroupNorm (GN), and Vision Transformers (ViT) with LayerNorm (LN). For ViTs, we employ one trained from scratch (noted as TS), and one pre-trained on a larger dataset and fine-tuned on dataset at hand (noted as FT). All fine-tuned ViTs are pre-trained on ImageNet-21K before fine-tuning. 

% For CIFAR-10, our model suite includes ResNet-18 (BN), ResNet-26 (GN), a ViT-small (TS), and a ViT-B (FT).  
% For CIFAR-100, our model suite includes ResNet-50 (BN), ResNet-50 (GN), a ViT-small (TS), and a ViT-B (FT).  
% For ImageNet-1K, our model suite includes ResNet-50 (BN), ResNet-50 (GN), a ViT-B (TS) using Sharpness-Aware Minimization (SAM), and a ViT-B (FT).
% Each set of experiments is evaluated over three random seeds to ensure statistical significance.
% \begin{itemize}
% \item CIFAR-10: ResNet-18 (BatchNorm), ResNet-26-GN (GroupNorm), two ViTs (LayerNorm): ViT-Patch4-32 trained from scratch and ViT-Base-Patch16-224 pre-trained on ImageNet21K and fine-tuned on CIFAR-10
% \item CIFAR-100: ResNet-50 (BatchNorm), ResNet-50-GN (GroupNorm), two ViTs (LayerNorm): ViT-Patch4-32 trained from scratch and ViT-Base-Patch16-224 pre-trained on ImageNet21K and fine-tuned on CIFAR-100
% \item ImageNet-1K: ResNet-50 (BatchNorm), ResNet-50-GN (GroupNorm), two ViTs (LayerNorm): ViT-Base-Patch16-224 trained from scratch with Sharpness-Aware Minimization and ViT-Base-Patch16-224 pre-trained on ImageNet21K and fine-tuned on Imagenet-1K
% \end{itemize}
We obtain pre-trained ResNet-26 (GN) weights from \cite{zhang_memo_2022} and ImageNet models from the PyTorch model zoo and the \texttt{timm} library \cite{rw2019timm}. ViT-Small models are trained from scratch using the approach described in \cite{yoshioka2024visiontransformers}.
For ResNet models on CIFAR-100, we used pre-trained ResNet-50 (BN) and ResNet-50 (GN) models and fine-tuned them on CIFAR-100, achieving test accuracies of 78.3\% and 72.6\% following the training methodology of \cite{kuangliu2020pytorchcifar}. 
For the ViT-B (FT) models, images were resized to 224×224. We fine-tuned these models that were pre-trained on ImageNet21K, reaching test accuracies of 98.6\% for CIFAR-10 and 90.5\% for CIFAR-100.

\subsection{Experiment Setup}
For experiments with AQR, we used 10,000 training samples from each dataset to estimate the source percentiles. These statistics are frozen for all test-time evaluations.
\label{sec:vit_layers}
We experimented with applying AQR transformation throughout the network as well as in just the top half. For more complex datasets (Imagenet) with ViT, we found it was beneficial only to include it in the top half. See ablations in Appendix \ref{app:ablation}. We hypothesize that for ViT, the layer normalization is placed at the beginning of the block without influencing the residual stream. Therefore, placing AQR after layer normalization can leave the residual stream unadapted. This could pass noisy input to subsequent layers through the residual stream, even with AQR placed within the block; thus, limiting the use of AQR towards the top of the network can be beneficial. %In particular for our 
%For ImageNet-trained Vision Transformers, we adapt only the top half of layers. This approach avoids disrupting the stable feature representations learned during large-scale training while allowing effective adjustments in the deeper layers.
%In contrast, for CIFAR-trained ViTs and ResNets, we adapt all layers since small-dataset training produces less stable representations that benefit from recalibration.
% We obtain pre-trained weights for ResNet-26-GN from \cite{zhang_memo_2022} and ImageNet models from the PyTorch model zoo and the \texttt{timm} library \cite{rw2019timm}. ViT-Patch4-32 models are trained from scratch using \cite{yoshioka2024visiontransformers}.
% For CIFAR-100, we fine-tuned ResNet-50 and ResNet-50-GN achieving test accuracies of 78.3\% and 72.6\%. Training followed a cosine learning rate schedule over 200 epochs, using SGD with momentum 0.9, weight decay $5×10-45\times10^{-4}$, and initial learning rate 0.1. Batch size 256 was used with standard data augmentation (random crop with padding 4, horizontal flip) and per-channel normalization. Implementation based on \cite{kuangliu2020pytorchcifar}.

\textbf{Episodic Setting of TTA:}
We compare AQR to TTN, TENT, and SAR, along with the unadapted source models. Since AQR is designed for stateless test-time use, we evaluate all methods in an \textbf{offline} (episodic) mode. Specifically, the inference on each batch is independent; This is an ideal setting for when distribution shift is highly variable.

For TENT and SAR, we reset model weights after every batch to their pre-trained values. Within each batch, a single forward–backward step is used for adaptation, followed by a second forward pass for evaluation. This approach ensures these parameter-updating methods have the opportunity to adapt before evaluation. This explains why in some cases, the models that have been adapted with SAR or TENT perform worse than the model that has not been adapted. TTN is applied as originally proposed. However, TTN is only applicable to model architectures that use BatchNorm, therefore TTN is excluded from experiments with ViTs or with ResNets that use GroupNorm.

\subsection{Main Results}

Our experimental evaluation shows the effectiveness of AQR across multiple architectures, datasets, and corruption severities. On the challenging ImageNet-C dataset, Table~\ref{tab:imagenet_compact} highlights the key ImageNet-C experiments (severity 3, batch size 128), where AQR consistently surpasses TTN, TENT, and SAR across corruption types and architectures. The detailed experimental results, including comprehensive per-dataset summaries across all batch sizes and severity levels, are provided in the Appendix \ref{app:detailed_tables} (Tables \ref{tab:cifar10_summary}, \ref{tab:cifar100_summary}, \ref{tab:imagenet_summary}).

\textbf{Cross-Dataset Performance}
Our method consistently outperforms existing approaches across all three benchmark datasets. As illustrated in Figure \ref{fig:severity_comparison}, AQR's advantage increases substantially with corruption severity, demonstrating its robustness to challenging test conditions. At severity level 1 (mild corruptions), AQR maintains competitive performance with existing methods. However, as corruption intensity increases to severity 3 and 5, AQR provides increasing improvements over baselines. Since TTN is not applicable to all architectures, it is omitted from this figure; a corresponding plot limited to ResNet models with BatchNorm is provided in the Appendix (Figure \ref{fig:severity_bn}).

\textbf{Architecture-Specific Analysis}
Figure \ref{fig:architecture_comparison} reveals that the effectiveness of AQR spans diverse architectural designs. For Vision Transformers, which do not employ batch normalization layers, AQR provides improvements through its quantile-based normalization approach.
ResNet architectures benefit significantly from AQR across both batch normalization and group normalization variants.

\textbf{Batch Size Effects}
Our analysis reveals that AQR's advantages are particularly pronounced with larger batch sizes, which provide more robust estimation of quantiles for the incoming batch. At batch size 512, the performance gaps between AQR and baselines are consistently larger than at batch size 128 across all datasets. 

This observation adds more evidence for the need to use the tail calibration strategy, as the batch size gets smaller and insufficient for estimating the target distribution tails.

\textbf{Robustness Across Corruption Types} The per-corruption analysis in Table \ref{tab:imagenet_compact} shows AQR’s consistent effectiveness across diverse corruption categories. It performs particularly well on noise-based corruptions (Gaussian, impulse, shot noise) and distortion-based ones (elastic transform, pixelate). For digital corruptions (JPEG compression, spatter) and weather distortions (fog, brightness, contrast), AQR remains competitive.

\begin{figure*}[t]
    \centering
    \includegraphics[width=1.02\linewidth]{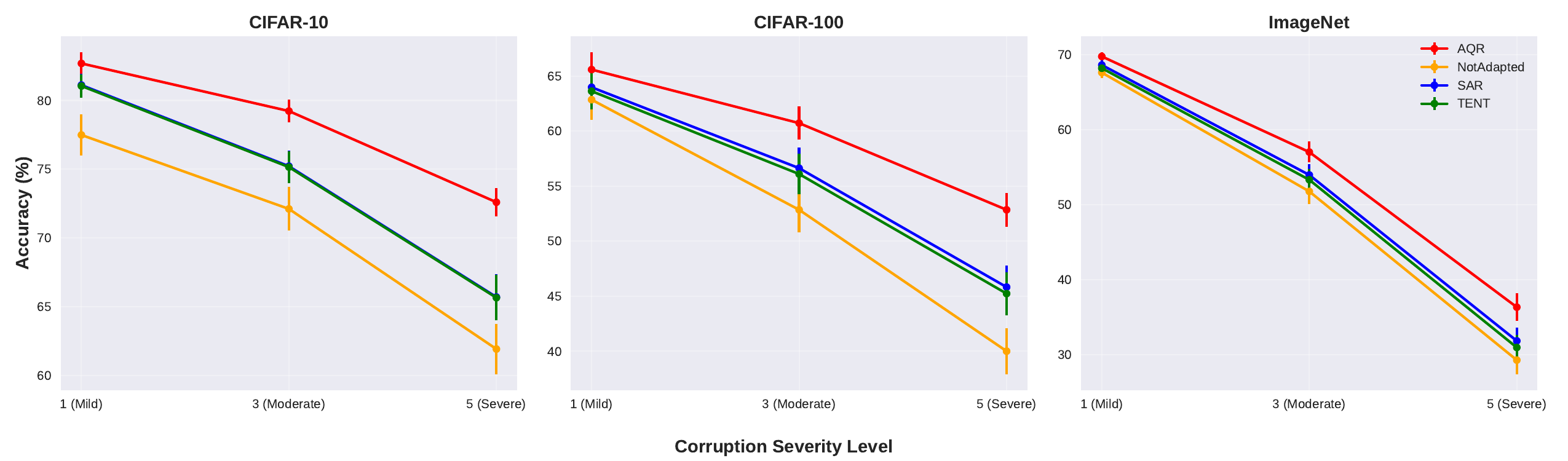}
    \caption{Performance comparison across corruption severity levels. AQR consistently outperforms baseline methods on all datasets, with larger performance gains at higher severities. Results averaged across all corruption types, batch sizes, and architectures. Error bars represent the standard error of the mean for different experimental conditions.}
    \label{fig:severity_comparison}
\end{figure*}

\subsection{Ablation on Mechanisms to Calibrate the Tail}
\label{sec:tails}

We conducted an ablation study to evaluate different strategies for calibrating extreme tails. The standard AQR method serves as our baseline approach. 

\textbf{Average Sample Tails.} We tested our proposed enhancement, AQR with the average of samples of tails, which uses the sampling technique described in the previous section to better estimate extreme percentiles. Specifically, we sample batches of 100 points, compute the minimum ($p_0$) and maximum ($p_100$) for each batch, repeat this 1{,}000 times, and average the results to obtain stable tail estimates. Subsequently, at inference time, these estimated values will be used to perform the recalibration.

\textbf{AQR without Tail Adaptation.} We also explored a simpler alternative: AQR with no tail adaptation, which we denote \textit{Not Calibrated}. Since the extreme ends ($[p_0,p_1]$ and $[p_{99},p_{100}]$) contain only 2\% of the data, we tested whether simply not adapting these regions would be effective. Expressed as 
\begin{equation}
\text{AQR}(x) = 
\begin{cases} 
x & x < p_1^T \\
x & x \geq p_{99}^T
\end{cases}.
\label{eq:na}
\end{equation}
This approach leaves values below the first percentile or above the 99th percentile unchanged.

\textbf{Clipping.} 
 This approach implements simple thresholding. Expressed as
 \begin{equation}
\text{AQR}(x) = 
\begin{cases} 
p_1^T & x < p_1^T \\
 p_{99}^T  & x \geq p_{99}^T 
\end{cases},
\label{eq:clipping}
\end{equation}
it sets extreme values to fixed boundaries - values below the 1st percentile are set exactly to the 1st percentile value, and values above the 99th percentile are set to the 99th percentile value. 

\textbf{Gaussian Estimation}. Another approach of calibrating the tails, which we call \textit{Gaussian Estimation}, assumes both source and target data follow normal distributions. Instead of using unreliable minimum and maximum values from small batches, it fits Gaussian curves to the data. It then uses these fitted curves to predict what the true $0^{th}$ and $100^{th}$ percentiles should theoretically be. A complete mapping formula of this strategy is provided in Appendix~\ref{app:gaussian-tail}.
\textbf{Interval Estimation}. This approach aims to estimate the magnitude of extreme intervals ($\Delta_0^T$  and $\Delta_{99}^S$). Rather than relying on percentile intervals for scaling transformations, it uses the standard deviation of the distribution. The \textit{interval estimation} can be expressed as follows
\begin{table}
\centering
\caption{Classification accuracy (\%) of different tail calibration strategies using ResNet50 on various ImageNet datasets. Results are averaged over 3 random seeds with the best results in bold.}
\label{tab:resnet50_handle_tail}
\setlength{\tabcolsep}{4pt}
\begin{tabular}{l|ccc|ccc}
\toprule
\multirow{2}{*}{\textbf{Tail Calibration Strategy}} & \multicolumn{2}{c}{\textbf{Batch Size}} \\
\cmidrule(lr){2-3}
 & 128 & 512 \\
\midrule
AQR (standard) & 30.8±16.3 & 33.7±16.3  \\
Average Sample Tails & \textbf{33.7±16.6} & \textbf{34.6±16.1}\\
Gaussian Estimation & 33.6±16.6 & 33.5±16.1 \\
Not Calibrated & 29.7±16.6 & 33.3±16.3 \\
Interval Estimation & 29.3±15.6 & 30.8±15.2\\
Clipping & 3.4±4.8 & 3.6±4.7\\
\bottomrule
\end{tabular}
\end{table}

% \subsection{Interval Estimation}
\begin{equation}
\text{AQR}(x) = 
\begin{cases} 
\left( \dfrac{a - p_0^T}{std(X)} \cdot \left( \gamma \right) \right) + p_0^S & x < p_1^T \\
\left( \dfrac{a - p_{99}^T}{std(X)} \cdot \left( \gamma \right) \right) + p_{99}^S & x \geq p_{99}^T
\end{cases},
\label{eq:interval}
\end{equation}
where $X$ is all points in a specific channel of a specific layer of a batch at test time and ${std(X)}$ is the standard deviation of $X$. The idea is that standard deviation provides a more stable measure of data spread and applies the remapping rule of Eq.~\ref{eq:interval}.

\begin{figure*}[!htbp]
    \centering
    \includegraphics[width=\textwidth]{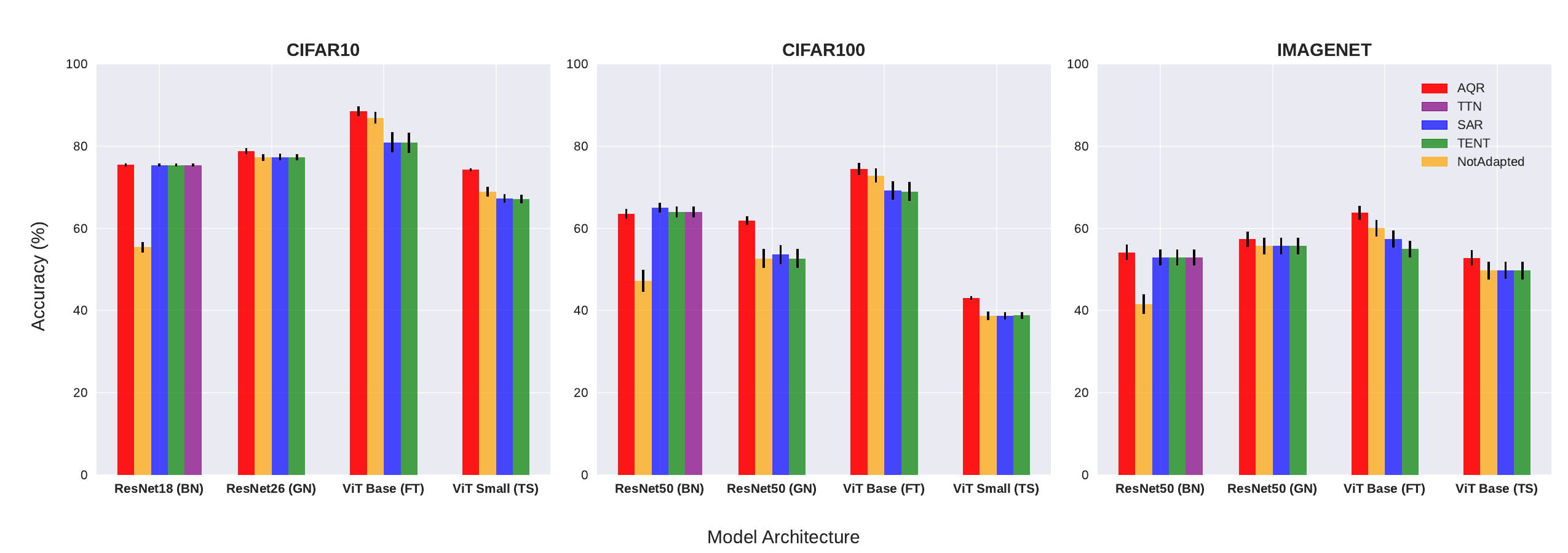}
    \caption{Architecture-specific performance comparison at corruption severity level 3. AQR demonstrates consistent improvements across diverse architectures on three datasets (CIFAR-10-C, CIFAR-100-C, ImageNet-C), including ResNets with different normalization schemes (BN, GN) and ViTs (LN). Error bars represent standard error across corruption types and batch sizes}
    \label{fig:architecture_comparison}
\end{figure*}

\begin{table*}[!htbp]
\caption{Impact of percentile granularity on AQR performance. Comparison between using 101 percentiles (standard AQR) versus 11 percentiles on ResNet-50 (BN) on ImageNet-C. }
\centering
\begin{tabular}{l|cc|cc}
\toprule
\multirow{2}{*}{\textbf{Method}} & \multicolumn{2}{c|}{\textbf{Batch Size = 128}} & \multicolumn{2}{c}{\textbf{Batch Size = 512}} \\
\cmidrule(lr){2-3} \cmidrule(lr){4-5}
 & Severity 3 & Severity 5 & Severity 3 & Severity 5 \\
\midrule
AQR & 53.92 ± 11.79 & 33.72 ± 16.33 & 54.40 ± 11.69 & 34.19 ± 16.37 \\
AQR (10 qds) & 46.37 ± 13.06 & 27.00 ± 16.07 & 50.68 ± 12.54 & 29.77 ± 16.49 \\
\bottomrule
\end{tabular}
\label{tab:aqr_results}
\end{table*}

Table \ref{tab:resnet50_handle_tail} reports classification accuracies for several tail‑calibration strategies evaluated with ResNet‑50 (BN) on the ImageNet-C dataset, across two batch sizes on all corruption types. The results demonstrate that precise modeling of extreme percentiles is essential for robustness. \enquote{Average Sample Tails} attains the highest accuracy at both batch sizes, marginally outperforming \enquote{Gaussian Estimation} and considerably exceeding the standard AQR baseline; this gap is most pronounced at the smaller batch size (128), where tail statistics are most volatile. \enquote{Interval Estimation} offers only a modest benefit, whereas aggressive \enquote{Clipping} severely degrades performance. Collectively, these findings confirm that nuanced calibrating of the values at the tails of the distribution is required and suggest \enquote{Average Sample Tails} as the most dependable approach.

\subsection{Granularity of Percentiles}
We also performed an ablation study to investigate how the granularity of the percentiles affects AQR performance. To this end, we computed fewer percentiles at setup phase to understand the trade-off between computational complexity and the performance of adaptation. Specifically, we compared using 11 percentiles ($p_0,p_{10},p_{20},\dots,p_{100}$) and our standard 101 percentiles ($p_0,p_{1},p_{2},\dots,p_{100}$) on ResNet50 (BN) with ImageNet-C. The results show that finer granularity leads to better performance, with 101 percentiles achieving a much better accuracy. 
%However, the improvement from 51 to 101 percentiles is modest,  making smaller granularity of percentiles a viable option for computational-constrained environments

\section{Conclusion}
In this study, we introduced Adaptive Quantile Recalibration (AQR), a novel test-time adaptation approach that aligns the distributions of internal features between source and target domains through nonparametric quantile-based transformations. Our approach offers several key advantages: (1) it captures the complete shape of activation distributions rather than just mean and variance, enabling more effective adaptation for complex distribution patterns; (2) our tail calibration strategy effectively handles the challenges of estimating distribution extremes with varying batch sizes; (3) it maintains stability during extended test sessions by using fixed source distribution statistics as reference points. Experiments on CIFAR-10-C, CIFAR-100-C and ImageNet-C across multiple architectures demonstrate that AQR outperforms state-of-the-art TTA methods.
However, while AQR provides stability through fixed reference statistics, this design choice means it processes batches independently rather than accumulating knowledge across sequential batches, which could potentially enhance adaptation in some online scenarios. Future work could explore combining AQR with other TTA methods and extending AQR to online settings while preserving its stability advantages.

\section*{Acknowledgment}
We acknowledge funding from FRQNT-NOVA grant %2023-NOVA- 329759.

%%%%%%%%%%%%%
\appendix
%\clearpage
\newtheorem{theorem}{Theorem}
{\small
\bibliographystyle{ieee_fullname}
\bibliography{egbib}

@String(ECCV= {Eur. Conf. Comput. Vis.})

@String(ICLR = {Int. Conf. Learn. Represent.})

@String(AAAI = {AAAI})

@String(ECCV  = {ECCV})

@String(ICLR  = {ICLR})

@misc{amodio_neuron_editting,
	title = {Out-of-sample extrapolation with neuron editing},
	author = {Amodio, Matthew and van Dijk, David and Montgomery, Ruth and Wolf, Guy and Krishnaswamy, Smita},
	year = {2019},
	eprint = {1805.12198},
	archivePrefix = {arXiv},
	primaryClass = {q-bio.QM}
}

@article{batchnorm,
  title={Batch normalization: Accelerating deep network training by reducing internal covariate shift. CoRR abs/1502.03167},
  author={Ioffe, Sergey and Szegedy, Christian},
  journal={arXiv preprint arXiv:1502.03167},
  year={2015}
}

@article{layernorm,
  title={Layer normalization. CoRR abs/1607.06450 (2016)},
  author={Ba, Lei Jimmy and Kiros, Ryan and Hinton, Geoffrey E},
  journal={arXiv preprint arXiv:1607.06450},
  year={2016}
}

@inproceedings{groupnorm,
  title={Group normalization},
  author={Wu, Yuxin and He, Kaiming},
  booktitle={Proceedings of the European Conference on Computer Vision (ECCV)},
  pages={3--19},
  year={2018}
}

@inproceedings{TENT,
  author    = {Dequan Wang and
               Evan Shelhamer and
               Shaoteng Liu and
               Bruno A. Olshausen and
               Trevor Darrell},
  title     = {Tent: Fully Test-Time Adaptation by Entropy Minimization},
  booktitle = {9th International Conference on Learning Representations, {ICLR} 2021,
               Virtual Event, Austria, May 3-7, 2021},
  publisher = {OpenReview.net},
  year      = {2021},
  url       = {https://openreview.net/forum?id=uXl3bZLkr3c},
  bibsource = {dblp computer science bibliography, https://dblp.org}
}

@misc{nado2020evaluating,
	title = {Evaluating prediction-time batch normalization for robustness under covariate shift},
	author = {Nado, Zachary and Padhy, Shreyas and Sculley, D and D'Amour, Alexander and Lakshminarayanan, Balaji and Snoek, Jasper},
	year = {2020},
	eprint = {2006.10963},
	archivePrefix = {arXiv}
}

@article{schneider2020improving,
  title={Improving robustness against common corruptions by covariate shift adaptation},
  author={Schneider, Steffen and Rusak, Evgenia and Eck, Luisa and Bringmann, Oliver and Brendel, Wieland and Bethge, Matthias},
  journal={Advances in Neural Information Processing Systems},
  volume={33},
  year={2020}
}

@misc{niu_towards_2023,
	title = {Towards stable test-time adaptation in dynamic wild world},
	author = {Niu, Shuaicheng and Wu, Jiaxiang and Zhang, Yifan and Wen, Zhiquan and Chen, Yaofo and Zhao, Peilin and Tan, Mingkui},
	year = {2023},
	eprint = {2302.12400},
	archivePrefix = {arXiv},
	primaryClass = {cs.CV}
}

@misc{yoshioka2024visiontransformers,
	author = {Yoshioka, Kentaro},
	title = {vision-transformers-cifar10: Training vision transformers ({ViT}) and related models on {CIFAR-10}},
	year = {2024},
	howpublished = {\url{https://github.com/kentaroy47/vision-transformers-cifar10}},
	note = {GitHub repository}
}

@misc{zhang_memo_2022,
	title = {{MEMO}: {Test} {Time} {Robustness} via {Adaptation} and {Augmentation}},
	shorttitle = {{MEMO}},
	url = {http://arxiv.org/abs/2110.09506},
	doi = {10.48550/arXiv.2110.09506},
	urldate = {2025-02-16},
	publisher = {arXiv},
	author = {Zhang, Marvin and Levine, Sergey and Finn, Chelsea},
	month = oct,
	year = {2022},
}

@techreport{krizhevsky2009learning,
    author = {Alex Krizhevsky},
    title = {Learning multiple layers of features from tiny images},
    institution = {University of Toronto},
    year = {2009}
}

@article{hendrycks2019benchmarking,
  title={Benchmarking neural network robustness to common corruptions and perturbations},
  author={Hendrycks, Dan and Dietterich, Thomas},
  journal={arXiv preprint arXiv:1903.12261},
  year={2019}
}

@article{li2016revisiting,
  title={Revisiting batch normalization for practical domain adaptation},
  author={Li, Yanghao and Wang, Naiyan and Shi, Jianping and Liu, Jiaying and Hou, Xiaodi},
  journal={arXiv preprint arXiv:1603.04779},
  year={2016}
}

@misc{vianna2024hybridttn,
      title={Channel-Selective Normalization for Label-Shift Robust Test-Time Adaptation}, 
      author={Pedro Vianna and Muawiz Chaudhary and Paria Mehrbod and An Tang and Guy Cloutier and Guy Wolf and Michael Eickenberg and Eugene Belilovsky},
      year={2024},
      eprint={2402.04958},
      archivePrefix={arXiv},
      primaryClass={cs.CV},
      url={https://arxiv.org/abs/2402.04958}, 
}

@misc{wang2024otta,
      title={In Search of Lost Online Test-time Adaptation: A Survey}, 
      author={Zixin Wang and Yadan Luo and Liang Zheng and Zhuoxiao Chen and Sen Wang and Zi Huang},
      year={2024},
      eprint={2310.20199},
      archivePrefix={arXiv},
      primaryClass={cs.AI},
      url={https://arxiv.org/abs/2310.20199}, 
}

@inproceedings{
lee2024cmf,
title={Continual Momentum Filtering on Parameter Space for Online Test-time Adaptation},
author={Jae-Hong Lee and Joon-Hyuk Chang},
booktitle={The Twelfth International Conference on Learning Representations},
year={2024},
url={https://openreview.net/forum?id=BllUWdpIOA}
}

@inproceedings{
lee2024deyo,
title={Entropy is not Enough for Test-Time Adaptation: From the Perspective of Disentangled Factors},
author={Jonghyun Lee and Dahuin Jung and Saehyung Lee and Junsung Park and Juhyeon Shin and Uiwon Hwang and Sungroh Yoon},
booktitle={The Twelfth International Conference on Learning Representations},
year={2024},
url={https://openreview.net/forum?id=9w3iw8wDuE}
}

@misc{marsden2023roid,
      title={Universal Test-time Adaptation through Weight Ensembling, Diversity Weighting, and Prior Correction}, 
      author={Robert A. Marsden and Mario Döbler and Bin Yang},
      year={2023},
      eprint={2306.00650},
      archivePrefix={arXiv},
      primaryClass={cs.CV},
      url={https://arxiv.org/abs/2306.00650}, 
}

@misc{kuangliu2020pytorchcifar,
  author       = {Liu, Kuang},
  title        = {pytorch-cifar},
  year         = {2020},
  howpublished = {\url{https://github.com/kuangliu/pytorch-cifar}},
  note         = {Accessed: 2024-07-18}
}

@article{liang2025comprehensive,
  title={A comprehensive survey on test-time adaptation under distribution shifts},
  author={Liang, Jian and He, Ran and Tan, Tieniu},
  journal={International Journal of Computer Vision},
  volume={133},
  number={1},
  pages={31--64},
  year={2025},
  publisher={Springer}
}

@inproceedings{yuan2023robust,
  title={Robust test-time adaptation in dynamic scenarios},
  author={Yuan, Longhui and Xie, Binhui and Li, Shuang},
  booktitle={Proceedings of the IEEE/CVF Conference on Computer Vision and Pattern Recognition},
  pages={15922--15932},
  year={2023}
}

@inproceedings{wu2024test,
author = {Wu, Yanan and Chi, Zhixiang and Wang, Yang and Plataniotis, Konstantinos N. and Feng, Songhe},
title = {Test-time domain adaptation by learning domain-aware batch normalization},
year = {2024},
isbn = {978-1-57735-887-9},
publisher = {AAAI Press},
url = {https://doi.org/10.1609/aaai.v38i14.29527},
doi = {10.1609/aaai.v38i14.29527},
booktitle = {Proceedings of the Thirty-Eighth AAAI Conference on Artificial Intelligence and Thirty-Sixth Conference on Innovative Applications of Artificial Intelligence and Fourteenth Symposium on Educational Advances in Artificial Intelligence},
articleno = {1779},
numpages = {9},
series = {AAAI'24/IAAI'24/EAAI'24}
}

@INPROCEEDINGS{lessBN,
  author={Gu, Xin},
  booktitle={2024 3rd International Conference on Cloud Computing, Big Data Application and Software Engineering (CBASE)}, 
  title={Less is More: Revisiting the Role of Batch Normalization in Test-time Adaptation}, 
  year={2024},
  volume={},
  number={},
  pages={49-53},
  doi={10.1109/CBASE64041.2024.10824310}}

@article{lim2023ttn,
  title={Ttn: A domain-shift aware batch normalization in test-time adaptation},
  author={Lim, Hyesu and Kim, Byeonggeun and Choo, Jaegul and Choi, Sungha},
  journal={arXiv preprint arXiv:2302.05155},
  year={2023}
}

@misc{rw2019timm,
  author = {Ross Wightman},
  title = {PyTorch Image Models},
  year = {2019},
  publisher = {GitHub},
  journal = {GitHub repository},
  doi = {10.5281/zenodo.4414861},
  howpublished = {\url{https://github.com/rwightman/pytorch-image-models}}
}
}

\newpage
\appendix
\onecolumn
\label{app:details}

% \section{TTA methods comparison by corruption}

% % the big table with all the corruptions

\newpage
\newtheorem{lemma}{Lemma}
\section{Detailed Results}
\label{app:detailed_tables}
Here we present detailed experimental results in three tables for each dataset (CIFAR10-C, CIFAR100-C and ImageNet-C). The results for AQR and its baselines are specified for two different batch sizes and three corruption severity levels.

\begin{table*}[htbp]
\centering
\footnotesize
\caption{Summary of CIFAR10 experiments. Average accuracy (\%) for each batch size and severity level. Values show mean±std across corruptions and seeds.}
\label{tab:cifar10_summary}
\resizebox{0.7\textwidth}{!}{%
\begin{tabular}{l|ccc|ccc}
\toprule
\textbf{Method} & \multicolumn{3}{c|}{\textbf{Batch Size 128}} & \multicolumn{3}{c}{\textbf{Batch Size 512}} \\
 & \textbf{Sev 1} & \textbf{Sev 3} & \textbf{Sev 5} & \textbf{Sev 1} & \textbf{Sev 3} & \textbf{Sev 5} \\
\midrule
ResNet18 (BN) & 59.5{\scriptsize ±4.4} & 55.4{\scriptsize ±8.0} & 46.8{\scriptsize ±8.3} & 59.5{\scriptsize ±4.4} & 55.4{\scriptsize ±8.0} & 46.8{\scriptsize ±8.3} \\
\hspace{0.5em}TTN & \textbf{78.0}{\scriptsize ±2.0} & 75.2{\scriptsize ±2.4} & 70.2{\scriptsize ±4.4} & \textbf{78.3}{\scriptsize ±2.0} & 75.5{\scriptsize ±2.4} & 70.6{\scriptsize ±4.4} \\
\hspace{0.5em}TENT & \textbf{78.0}{\scriptsize ±2.0} & 75.2{\scriptsize ±2.4} & 70.2{\scriptsize ±4.4} & \textbf{78.3}{\scriptsize ±2.0} & 75.5{\scriptsize ±2.4} & 70.6{\scriptsize ±4.4} \\
\hspace{0.5em}SAR & \textbf{78.0}{\scriptsize ±2.0} & 75.2{\scriptsize ±2.4} & 70.2{\scriptsize ±4.4} & \textbf{78.3}{\scriptsize ±2.0} & 75.5{\scriptsize ±2.4} & 70.6{\scriptsize ±4.4} \\
\hspace{0.5em}AQR & 77.9{\scriptsize ±2.0} & \textbf{75.3}{\scriptsize ±2.4} & \textbf{70.6}{\scriptsize ±4.2} & \textbf{78.3}{\scriptsize ±2.0} & \textbf{75.6}{\scriptsize ±2.4} & \textbf{70.9}{\scriptsize ±4.2} \\
\midrule
ResNet26 (GN) & \textbf{81.8}{\scriptsize ±4.3} & 77.3{\scriptsize ±5.0} & 68.2{\scriptsize ±7.0} & 81.8{\scriptsize ±4.3} & 77.3{\scriptsize ±5.0} & 68.2{\scriptsize ±7.0} \\
\hspace{0.5em}TENT & \textbf{81.8}{\scriptsize ±4.3} & 77.3{\scriptsize ±5.0} & 68.2{\scriptsize ±7.0} & 81.8{\scriptsize ±4.3} & 77.3{\scriptsize ±5.0} & 68.2{\scriptsize ±7.0} \\
\hspace{0.5em}SAR & \textbf{81.8}{\scriptsize ±4.3} & 77.3{\scriptsize ±5.0} & 68.2{\scriptsize ±7.0} & 81.8{\scriptsize ±4.3} & 77.3{\scriptsize ±5.0} & 68.2{\scriptsize ±7.0} \\
\hspace{0.5em}AQR & 81.7{\scriptsize ±3.7} & \textbf{78.6}{\scriptsize ±4.4} & \textbf{73.2}{\scriptsize ±5.7} & \textbf{81.9}{\scriptsize ±3.6} & \textbf{78.9}{\scriptsize ±4.4} & \textbf{73.6}{\scriptsize ±5.6} \\
\midrule
ViT Base (FT) & 92.9{\scriptsize ±5.1} & 86.9{\scriptsize ±8.8} & 74.0{\scriptsize ±15.8} & 92.9{\scriptsize ±5.1} & 86.9{\scriptsize ±8.8} & 74.0{\scriptsize ±15.8} \\
\hspace{0.5em}TENT & 89.9{\scriptsize ±6.9} & 80.8{\scriptsize ±15.1} & 66.4{\scriptsize ±22.4} & 90.0{\scriptsize ±6.9} & 80.8{\scriptsize ±15.1} & 66.4{\scriptsize ±22.4} \\
\hspace{0.5em}SAR & 90.0{\scriptsize ±6.9} & 80.9{\scriptsize ±15.0} & 66.5{\scriptsize ±22.4} & 90.0{\scriptsize ±6.9} & 80.9{\scriptsize ±15.0} & 66.5{\scriptsize ±22.4} \\
\hspace{0.5em}AQR & \textbf{93.2}{\scriptsize ±3.9} & \textbf{88.4}{\scriptsize ±7.0} & \textbf{78.6}{\scriptsize ±11.2} & \textbf{93.4}{\scriptsize ±3.9} & \textbf{88.6}{\scriptsize ±7.0} & \textbf{78.9}{\scriptsize ±11.3} \\
\midrule
ViT Small (TS) & 75.7{\scriptsize ±3.5} & 68.9{\scriptsize ±7.4} & 58.7{\scriptsize ±14.9} & 75.7{\scriptsize ±3.5} & 68.9{\scriptsize ±7.4} & 58.7{\scriptsize ±14.9} \\
\hspace{0.5em}TENT & 74.3{\scriptsize ±3.1} & 67.2{\scriptsize ±6.5} & 57.7{\scriptsize ±12.9} & 74.3{\scriptsize ±3.2} & 67.1{\scriptsize ±6.6} & 57.6{\scriptsize ±12.9} \\
\hspace{0.5em}SAR & 74.5{\scriptsize ±3.1} & 67.3{\scriptsize ±6.6} & 57.7{\scriptsize ±13.0} & 74.6{\scriptsize ±3.2} & 67.3{\scriptsize ±6.6} & 57.7{\scriptsize ±13.0} \\
\hspace{0.5em}AQR & \textbf{77.4}{\scriptsize ±1.7} & \textbf{74.0}{\scriptsize ±2.5} & \textbf{67.3}{\scriptsize ±9.2} & \textbf{77.8}{\scriptsize ±1.7} & \textbf{74.4}{\scriptsize ±2.5} & \textbf{67.7}{\scriptsize ±9.3} \\
\bottomrule
\end{tabular}%
}
\end{table*}

\begin{table*}[htbp]
\centering
\footnotesize
\caption{Summary of CIFAR100 experiments. Average accuracy (\%) for each batch size and severity level. Values show mean±std across corruptions and seeds.}
\label{tab:cifar100_summary}
\resizebox{0.7\textwidth}{!}{%
\begin{tabular}{l|ccc|ccc}
\toprule
\textbf{Method} & \multicolumn{3}{c|}{\textbf{Batch Size 128}} & \multicolumn{3}{c}{\textbf{Batch Size 512}} \\
 & \textbf{Sev 1} & \textbf{Sev 3} & \textbf{Sev 5} & \textbf{Sev 1} & \textbf{Sev 3} & \textbf{Sev 5} \\
\midrule
ResNet50 (BN) & 63.3{\scriptsize ±13.6} & 47.2{\scriptsize ±16.4} & 31.1{\scriptsize ±17.3} & 63.3{\scriptsize ±13.6} & 47.2{\scriptsize ±16.4} & 31.1{\scriptsize ±17.3} \\
\hspace{0.5em}TTN & 68.5{\scriptsize ±4.9} & 63.5{\scriptsize ±8.0} & 56.4{\scriptsize ±10.1} & 69.5{\scriptsize ±4.9} & 64.5{\scriptsize ±8.0} & 57.2{\scriptsize ±10.2} \\
\hspace{0.5em}TENT & 68.4{\scriptsize ±4.8} & 63.5{\scriptsize ±7.8} & 56.3{\scriptsize ±10.0} & 69.5{\scriptsize ±4.8} & 64.5{\scriptsize ±8.0} & 57.2{\scriptsize ±10.1} \\
\hspace{0.5em}SAR & \textbf{69.7}{\scriptsize ±4.4} & \textbf{65.2}{\scriptsize ±6.9} & \textbf{58.7}{\scriptsize ±8.6} & \textbf{69.7}{\scriptsize ±4.8} & \textbf{64.8}{\scriptsize ±7.8} & 57.6{\scriptsize ±9.9} \\
\hspace{0.5em}AQR & 67.2{\scriptsize ±4.6} & 62.4{\scriptsize ±7.0} & 55.7{\scriptsize ±9.0} & 69.3{\scriptsize ±4.6} & 64.6{\scriptsize ±7.0} & \textbf{57.7}{\scriptsize ±9.2} \\
\midrule
ResNet50 (GN) & 62.1{\scriptsize ±10.1} & 52.6{\scriptsize ±14.2} & 39.5{\scriptsize ±14.4} & 62.1{\scriptsize ±10.1} & 52.6{\scriptsize ±14.2} & 39.5{\scriptsize ±14.4} \\
\hspace{0.5em}TENT & 62.2{\scriptsize ±10.1} & 52.7{\scriptsize ±14.2} & 39.6{\scriptsize ±14.4} & 62.2{\scriptsize ±10.1} & 52.7{\scriptsize ±14.2} & 39.6{\scriptsize ±14.4} \\
\hspace{0.5em}SAR & 63.3{\scriptsize ±9.5} & 54.2{\scriptsize ±14.3} & 40.7{\scriptsize ±15.5} & 62.4{\scriptsize ±10.0} & 53.0{\scriptsize ±14.1} & 39.8{\scriptsize ±14.5} \\
\hspace{0.5em}AQR & \textbf{66.0}{\scriptsize ±4.2} & \textbf{61.5}{\scriptsize ±6.5} & \textbf{54.8}{\scriptsize ±8.2} & \textbf{66.9}{\scriptsize ±4.1} & \textbf{62.4}{\scriptsize ±6.5} & \textbf{55.8}{\scriptsize ±8.2} \\
\midrule
ViT Base (FT) & 81.2{\scriptsize ±7.6} & 72.8{\scriptsize ±10.5} & 58.3{\scriptsize ±14.6} & 81.2{\scriptsize ±7.6} & 72.8{\scriptsize ±10.5} & 58.3{\scriptsize ±14.6} \\
\hspace{0.5em}TENT & 78.8{\scriptsize ±9.2} & 68.9{\scriptsize ±14.1} & 53.3{\scriptsize ±17.9} & 78.9{\scriptsize ±9.2} & 69.0{\scriptsize ±14.1} & 53.4{\scriptsize ±18.0} \\
\hspace{0.5em}SAR & 79.0{\scriptsize ±9.1} & 69.3{\scriptsize ±13.8} & 53.8{\scriptsize ±17.5} & 78.9{\scriptsize ±9.1} & 69.2{\scriptsize ±13.9} & 53.7{\scriptsize ±17.6} \\
\hspace{0.5em}AQR & \textbf{81.6}{\scriptsize ±6.3} & \textbf{74.3}{\scriptsize ±9.3} & \textbf{61.6}{\scriptsize ±12.7} & \textbf{81.8}{\scriptsize ±6.3} & \textbf{74.5}{\scriptsize ±9.2} & \textbf{61.8}{\scriptsize ±12.7} \\
\midrule
ViT Small (TS) & 44.8{\scriptsize ±3.9} & 38.7{\scriptsize ±6.2} & 31.1{\scriptsize ±10.3} & 44.8{\scriptsize ±3.9} & 38.7{\scriptsize ±6.2} & 31.1{\scriptsize ±10.3} \\
\hspace{0.5em}TENT & 44.6{\scriptsize ±3.5} & 38.8{\scriptsize ±5.4} & 31.2{\scriptsize ±10.4} & 44.5{\scriptsize ±3.6} & 38.8{\scriptsize ±5.3} & 31.2{\scriptsize ±10.4} \\
\hspace{0.5em}SAR & 44.5{\scriptsize ±3.6} & 38.7{\scriptsize ±5.4} & 31.1{\scriptsize ±10.3} & 44.4{\scriptsize ±3.5} & 38.7{\scriptsize ±5.3} & 31.1{\scriptsize ±10.3} \\
\hspace{0.5em}AQR & \textbf{45.8}{\scriptsize ±2.4} & \textbf{42.9}{\scriptsize ±3.1} & \textbf{37.6}{\scriptsize ±8.0} & \textbf{46.1}{\scriptsize ±2.4} & \textbf{43.2}{\scriptsize ±3.1} & \textbf{37.9}{\scriptsize ±8.1} \\
\bottomrule
\end{tabular}%
}
\end{table*}
\clearpage
\begin{table*}[htbp]
\centering
\footnotesize
\caption{Summary of ImageNet experiments. Average accuracy (\%) for each batch size and severity level. Values show mean±std across corruptions and seeds.}
\label{tab:imagenet_summary}
\resizebox{0.7\textwidth}{!}{%
\begin{tabular}{l|ccc|ccc}
\toprule
\textbf{Method} & \multicolumn{3}{c|}{\textbf{Batch Size 128}} & \multicolumn{3}{c}{\textbf{Batch Size 512}} \\
 & \textbf{Sev 1} & \textbf{Sev 3} & \textbf{Sev 5} & \textbf{Sev 1} & \textbf{Sev 3} & \textbf{Sev 5} \\
\midrule
ResNet50 (BN) & 61.9{\scriptsize ±6.6} & 41.6{\scriptsize ±14.6} & 19.4{\scriptsize ±14.4} & 61.9{\scriptsize ±6.6} & 41.6{\scriptsize ±14.6} & 19.4{\scriptsize ±14.4} \\
\hspace{0.5em}TTN & 67.3{\scriptsize ±4.4} & 52.7{\scriptsize ±12.0} & 32.5{\scriptsize ±16.0} & 67.7{\scriptsize ±4.4} & 53.1{\scriptsize ±12.0} & 32.9{\scriptsize ±16.1} \\
\hspace{0.5em}TENT & 67.4{\scriptsize ±4.4} & 52.7{\scriptsize ±12.0} & 32.6{\scriptsize ±16.0} & 67.7{\scriptsize ±4.4} & 53.1{\scriptsize ±11.9} & 32.9{\scriptsize ±16.1} \\
\hspace{0.5em}SAR & 67.4{\scriptsize ±4.4} & 52.8{\scriptsize ±11.9} & 32.7{\scriptsize ±16.0} & 67.7{\scriptsize ±4.4} & 53.2{\scriptsize ±11.9} & 33.0{\scriptsize ±16.1} \\
\hspace{0.5em}AQR & \textbf{67.4}{\scriptsize ±3.7} & \textbf{53.9}{\scriptsize ±11.9} & \textbf{33.7}{\scriptsize ±16.4} & \textbf{67.8}{\scriptsize ±3.7} & \textbf{54.4}{\scriptsize ±11.8} & \textbf{34.2}{\scriptsize ±16.5} \\
\midrule
ResNet50 (GN) & \textbf{70.6}{\scriptsize ±4.8} & 55.7{\scriptsize ±12.5} & 32.6{\scriptsize ±16.3} & \textbf{70.6}{\scriptsize ±4.8} & 55.7{\scriptsize ±12.5} & 32.6{\scriptsize ±16.3} \\
\hspace{0.5em}TENT & \textbf{70.6}{\scriptsize ±4.8} & 55.7{\scriptsize ±12.5} & 32.6{\scriptsize ±16.3} & 70.6{\scriptsize ±4.8} & 55.7{\scriptsize ±12.5} & 32.6{\scriptsize ±16.3} \\
\hspace{0.5em}SAR & \textbf{70.6}{\scriptsize ±4.8} & 55.7{\scriptsize ±12.5} & 32.6{\scriptsize ±16.3} & 70.6{\scriptsize ±4.8} & 55.7{\scriptsize ±12.5} & 32.6{\scriptsize ±16.3} \\
\hspace{0.5em}AQR & 69.0{\scriptsize ±3.6} & \textbf{56.6}{\scriptsize ±11.5} & \textbf{36.2}{\scriptsize ±16.7} & \textbf{70.8}{\scriptsize ±3.6} & \textbf{58.1}{\scriptsize ±11.6} & \textbf{37.3}{\scriptsize ±17.0} \\
\midrule
ViT Base (FT) & 71.3{\scriptsize ±5.0} & 60.0{\scriptsize ±12.2} & 40.6{\scriptsize ±12.6} & 71.3{\scriptsize ±5.0} & 60.0{\scriptsize ±12.2} & 40.6{\scriptsize ±12.6} \\
\hspace{0.5em}TENT & 68.0{\scriptsize ±5.6} & 54.9{\scriptsize ±12.4} & 34.1{\scriptsize ±11.1} & 68.0{\scriptsize ±5.6} & 54.9{\scriptsize ±12.4} & 34.1{\scriptsize ±11.1} \\
\hspace{0.5em}SAR & 68.0{\scriptsize ±5.5} & 55.0{\scriptsize ±12.4} & 34.1{\scriptsize ±11.1} & 71.3{\scriptsize ±4.9} & 59.8{\scriptsize ±12.1} & 40.7{\scriptsize ±12.5} \\
\hspace{0.5em}AQR & \textbf{73.9}{\scriptsize ±3.9} & \textbf{63.7}{\scriptsize ±10.3} & \textbf{45.6}{\scriptsize ±12.4} & \textbf{73.9}{\scriptsize ±3.9} & \textbf{63.8}{\scriptsize ±10.3} & \textbf{45.9}{\scriptsize ±12.2} \\
\midrule
ViT Base (TS) & 66.6{\scriptsize ±5.4} & 49.7{\scriptsize ±13.2} & 24.6{\scriptsize ±13.4} & 66.6{\scriptsize ±5.4} & 49.7{\scriptsize ±13.2} & 24.6{\scriptsize ±13.4} \\
\hspace{0.5em}TENT & 66.6{\scriptsize ±5.4} & 49.7{\scriptsize ±13.2} & 24.6{\scriptsize ±13.4} & 66.6{\scriptsize ±5.4} & 49.7{\scriptsize ±13.2} & 24.6{\scriptsize ±13.4} \\
\hspace{0.5em}SAR & 66.6{\scriptsize ±5.4} & 49.8{\scriptsize ±13.2} & 24.6{\scriptsize ±13.4} & 66.6{\scriptsize ±5.4} & 49.8{\scriptsize ±13.2} & 24.6{\scriptsize ±13.4} \\
\hspace{0.5em}AQR & \textbf{67.6}{\scriptsize ±4.7} & \textbf{52.8}{\scriptsize ±11.5} & \textbf{29.0}{\scriptsize ±13.6} & \textbf{67.5}{\scriptsize ±4.7} & \textbf{52.7}{\scriptsize ±11.5} & \textbf{28.9}{\scriptsize ±13.6} \\
\bottomrule
\end{tabular}%
}
\end{table*}
\section{Complementary Results for Performance on Model Architecture Variants}
Here we include results analogous to Figure \ref{fig:activation_diff_boxplot}, for ResNet (BN) architectures. This allows comparison of TTN against other approaches under the same architecture.

\begin{figure}[!htbp]
    \centering
    \includegraphics[width=\linewidth]{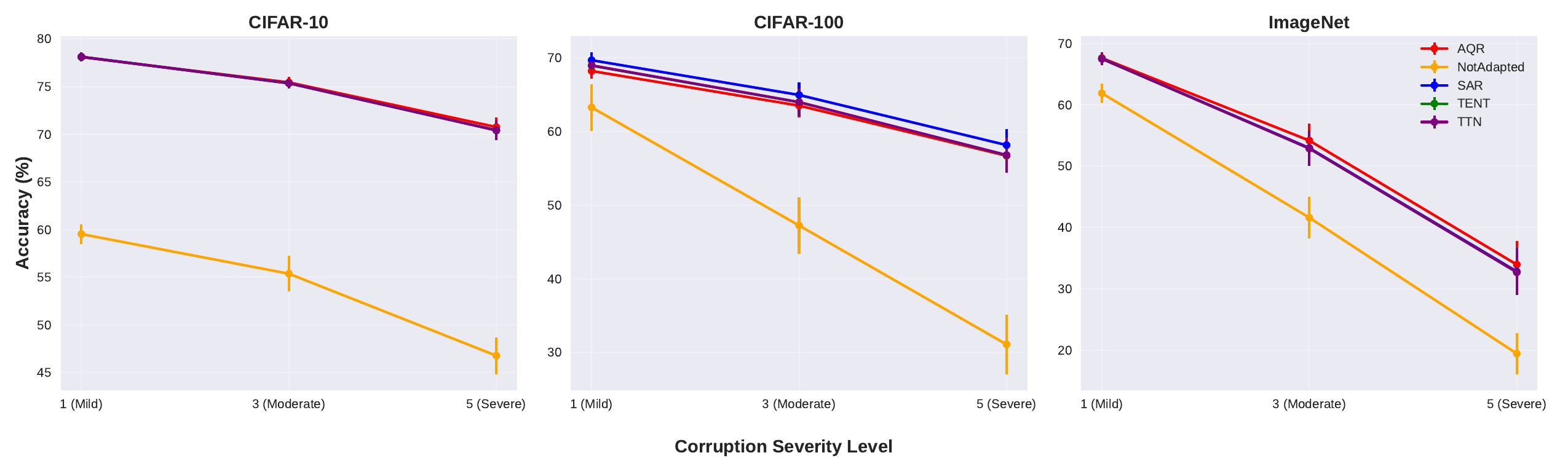}
    \caption{Performance comparison across corruption severity levels for ResNet (BN) models. Results averaged across all corruption types and batch sizes. Error bars represent the standard error of the mean for different experimental conditions.}
    \label{fig:severity_bn}
\end{figure}

% todo: gaussian estimation appendix
\clearpage
\section{Hyper-parameter Details in Experiments}

\begin{table}[h!]
\centering
\caption{Baseline methods hyperparameters used in our experiments.}
\begin{tabular}{l l l}
\toprule
Method & Hyperparameter & Value \\
\midrule
\multirow{3}{*}{TENT \& SAR} & Learning rate & $2.5\times 10^{-4}$ (SGD with momentum) \\
 & Learning rate (ViT) & $\frac{10^{-3}}{64}\times \text{batch size}$ (SGD with momentum) \\
 & Episodic setting & Offline/episodic (2 steps) \\
\midrule
SAR & $E_0$ threshold & $\ln(1000) \times 0.40$ \\
& Layer freezing strategy & Freeze top layers \\
 & ResNet50 (GN) frozen layers & layer 4 \\
 & ViT-Base (LN) frozen layers & blocks 9, 10, and 11\\

\bottomrule
\end{tabular}
\label{tab:baseline-hparams}
\end{table}

\begin{table}[h!]
\centering
\caption{AQR hyperparameters used in our experiments.}
\begin{tabular}{l l}
\toprule
Hyperparameter & Value \\
\midrule
\textbf{Source Data Configuration} & \\
Source samples for computing percentiles ($n_s$) & 10{,}000 \\
Quantile granularity & 101 (percentiles $p_0 \ldots p_{100}$) \\
\midrule
\textbf{Layer Selection Strategy} & \\
ViT-Base-patch16-224 (ImageNet-pre-trained/trained) & Top half of transformer blocks \\
ViT-Patch4-32 (CIFAR-10/100 from scratch) & All layers \\
ResNets & All layers \\
\midrule
\textbf{Tail Calibration Strategy} & \\
Tail calibration batch size & 100 \\
Tail calibration samples (repeats) & 1{,}000 \\
\bottomrule
\end{tabular}
\label{tab:aqr-hparams}
\end{table}

\section{Ablation on Blocks to Adapt}
\label{app:ablation}
% Vit base is pre-trained on imagenet21k finetuned on imagenet1k
% Vit small is trained from scratch on cifar10
% i can add vit base that is pre-trained on imagenet21k and fnietuned on cifar10
% and i can add resnet50 

The ablation study evaluates our method across different ViT architectures, where ViT-Base (FT) refers to models pre-trained on ImageNet21k and finetuned on ImageNet1k or CIFAR-10 respectively, while ViT-Small (TS) was trained from scratch on CIFAR-10, with results averaged across multiple random seeds and corruption types. 
\begin{table}[htbp]
    \centering
    \caption{Ablation study of applying AQR to different blocks in ViT architecture.}
\begin{tabular}{lllll}
\toprule
         Model &     Adapted Blocks & Fine-tuning/Training Dataset & Pretraining Dataset & Accuracy (\%) \\
\midrule
 ViT-Base (FT) &         All Blocks &                     CIFAR-10 &        ImageNet-21k &   87.47±7.78 \\
 ViT-Base (FT) & Bottom Half Blocks &                     CIFAR-10 &        ImageNet-21k &   82.4±16.01 \\
 ViT-Base (FT) &    Top Half Blocks &                     CIFAR-10 &        ImageNet-21k &    88.35±7.0 \\
 ViT-Base (FT) &         All Blocks &                  ImageNet-1k &        ImageNet-21k &  36.25±24.44 \\
 ViT-Base (FT) & Bottom Half Blocks &                  ImageNet-1k &        ImageNet-21k &   22.6±25.28 \\
 ViT-Base (FT) &    Top Half Blocks &                  ImageNet-1k &        ImageNet-21k &   63.7±10.33 \\
ViT-Small (TS) &         All Blocks &                     CIFAR-10 &                   - &   74.05±2.46 \\
ViT-Small (TS) & Bottom Half Blocks &                     CIFAR-10 &                   - &   68.93±7.41 \\
ViT-Small (TS) &    Top Half Blocks &                     CIFAR-10 &                   - &   68.93±7.41 \\
\bottomrule
\end{tabular}
    \label{tab:layer_ablation}
\end{table}

\section{Finite-Sample and Finite Quantile Discretization Analysis of AQR}
\label{app:finite-sample-aqr}
In this appendix, we provide a finite-sample analysis of the AQR method under practical constraints: finite sample sizes and finite quantile discretization. Our main objective is to establish theoretical guarantees showing that AQR's error converges to zero while TTN's error is lower bounded by a constant bias, therefore providing theoretical justification for AQR's superior empirical performance.

\paragraph{Main Theoretical Goal.}
For analytical clarity, we focus on a single neuron $i$ and establish per-neuron bounds. The analysis extends to the full network by summing over all neurons or applying union bounds across neurons.

We aim to prove that under finite-sample conditions with quantile discretization, AQR achieves (per neuron):
\begin{itemize}
    \item \textbf{Convergent error}: $\text{MSE}_i(T_{i,K,n}^{\text{AQR}}) \to 0$ as sample sizes $n_S, n_T$ and quantile resolution $K$ increase
    \item \textbf{Explicit convergence rates}: $O(K^{-4}) + O(n_S^{-1}) + O(n_T^{-1})$ with logarithmic factors
    \item \textbf{Superiority over TTN}: For sufficiently large $K, n_S, n_T$, we have $\text{MSE}_i(T_{i,K,n}^{\text{AQR}}) < \text{MSE}_i(T_i^{\text{TTN}})$
\end{itemize}

This analysis extends the idealized proof from the main text (Section~\ref{sec:anlysis_one_hidden_layer}) to more realistic settings with empirical estimation and computational constraints.

\paragraph{Setting and Notation.}
We adopt the one-hidden-layer MLP framework from the main text. Consider a network operating on the source domain with input $x \sim P$ on $\mathbb{R}^d$, weight matrix $W \in \mathbb{R}^{m \times d}$, bias $b \in \mathbb{R}^m$, and strictly increasing activation $\phi: \mathbb{R} \to \mathbb{R}$. The hidden representations are:
$$h^S = \phi(Wx + b) \in \mathbb{R}^m, \quad h^T = k(h^S)$$
where $k = (k_1, \ldots, k_m)$ with each $k_i: \mathbb{R} \to \mathbb{R}$ being strictly increasing (representing per-neuron corruptions).

For each neuron $i \in \{1, \ldots, m\}$, we denote:
\begin{itemize}
    \item $P_i$: source marginal of $h_i^S$ with CDF $F_{P_i}$, density $f_{P_i}$, quantile function $H_i = F_{P_i}^{-1}$
    \item $Q_i$: target marginal of $h_i^T$ with CDF $F_{Q_i}$
    \item Corruption relationship: $F_{Q_i}(u) = F_{P_i}(k_i^{-1}(u))$
\end{itemize}
The oracle AQR transformation is:
$T_{i,*}^{\text{AQR}}(z) := F_{P_i}^{-1}(F_{Q_i}(z)) = H_i(F_{Q_i}(z))$
which satisfies $T_{i,*}^{\text{AQR}}(h_i^T) = h_i^S$ (perfect recovery).
We measure adaptation quality by the per-neuron mean squared error:
$\text{MSE}_i(T) := \mathbb{E}[(T_i(h_i^T) - h_i^S)^2]$

\paragraph{Regularity Assumptions.}
For theoretical tractability, we impose these conditions on each source density $f_{P_i}$:
\begin{align}
0 < \underline{f}_i \leq f_{P_i}(x) \leq M_i < \infty \quad &\text{for all } x \in \text{supp}(P_i) \label{eq:density-bounds} \\
f_{P_i} \in C^1 \text{ with } \|f_{P_i}'\|_\infty \leq L_i \label{eq:smoothness}
\end{align}
These conditions ensure that the quantile function $H_i$ is well-behaved with bounded derivatives:
\begin{align}
\|H_i'\|_\infty \leq \frac{1}{\underline{f}_i}, \quad \|H_i''\|_\infty \leq \frac{L_i}{\underline{f}_i^3} \label{eq:quantile-bounds}
\end{align}
These bounds follow from the identities $H_i'(u) = 1/f_{P_i}(H_i(u))$ and $H_i''(u) = -f_{P_i}'(H_i(u))/f_{P_i}(H_i(u))^3$.

\paragraph{Practical AQR with Finite Samples and Quantiles.}
In practice, we only observe finite samples from both domains:
\begin{itemize}
    % \item Source samples: $\{h_{j}^S\}_{j=1}^{n_S} \stackrel{\text{iid}}{\sim} P$ 
    % \item Target samples: $\{h_{j}^T\}_{j=1}^{n_T} \stackrel{\text{iid}}{\sim} Q$ 
    \item Source samples: $\{h_{j}^S\}_{j=1}^{n_S} \stackrel{\text{iid}}{\sim} P$ with marginals $P_i$
    \item Target samples: $\{h_{j}^T\}_{j=1}^{n_T} \stackrel{\text{iid}}{\sim} Q$ with marginals $Q_i$
\end{itemize}
From these samples, we construct empirical CDFs $\widehat{F}_{P_i}$ and $\widehat{F}_{Q_i}$ for each neuron $i$.
For computational efficiency, we discretize the quantile transformation using $K$ uniform knots:
$$u_j := \frac{j}{K}, \quad j = 0, 1, \ldots, K.$$
The empirical source quantiles are:
$$\widehat{q}_{i,j} := \widehat{F}_{P_i}^{-1}(u_j).$$
We then define $\widetilde{H}_{i,K}$ as the piecewise-linear interpolant satisfying $\widetilde{H}_{i,K}(u_j) = \widehat{q}_{i,j}$ and linear interpolation on each interval $[u_{j-1}, u_j]$.
The practical AQR map becomes:
$$T_{i,K,n}^{\text{AQR}}(z) := \widetilde{H}_{i,K}(\widehat{F}_{Q_i}(z))$$
where $n = (n_S, n_T)$ represents the sample sizes.

\paragraph{Main Theoretical Result.}
Our main result provides explicit finite-sample error bounds for the practical AQR method:

\begin{theorem}[Finite-Sample AQR Error Bound]
\label{thm:finite-sample-aqr}
Under the regularity conditions~\eqref{eq:density-bounds}-\eqref{eq:smoothness}, for any neuron $i$ and any $\delta \in (0,1)$, with probability at least $1-2\delta$:
\begin{align}
\text{MSE}_i(T_{i,K,n}^{\text{AQR}}) &\leq 3\left(\frac{L_i}{8\underline{f}_i^3}\right)^2 K^{-4} + \frac{3}{\underline{f}_i^2}\varepsilon_S(\delta, n_S)^2 + \frac{3}{\underline{f}_i^2}\varepsilon_T(\delta, n_T)^2 \label{eq:main-bound}
\end{align}
where $\varepsilon_\bullet(\delta, n) := \sqrt{\frac{1}{2n}\log\frac{2}{\delta}}$.\\
In particular, $\text{MSE}_i(T_{i,K,n}^{\text{AQR}}) \to 0$ as $K, n_S, n_T \to \infty$ at rates $O(K^{-4})$, $O(n_S^{-1})$, and $O(n_T^{-1})$.
\end{theorem}

This result establishes that the error for AQR's mapping converges to zero, while the error for TTN's mapping maintains a constant positive bias whenever corruptions are non-affine.

\paragraph{Proof Strategy and Key Lemmas.}
Our proof decomposes the total error into three manageable components and bounds each using concentration inequalities and approximation theory. We present the analysis through a sequence of lemmas that build towards the main result.

\begin{lemma}[Error Decomposition]
\label{lem:error-decomp}
For any neuron $i$ and input $z \in \mathbb{R}$:
\begin{align}
T_{i,K,n}^{\text{AQR}}(z) - T_{i,*}^{\text{AQR}}(z) &= \underbrace{(\widetilde{H}_{i,K} - H_i)(\widehat{F}_{Q_i}(z))}_{\text{quantile estimation error}} + \underbrace{H_i(\widehat{F}_{Q_i}(z)) - H_i(F_{Q_i}(z))}_{\text{CDF estimation error}} \label{eq:error-decomp}
\end{align}
\end{lemma}
\begin{proof}
Considering practical AQR as $T_{i,K,n}^{\text{AQR}}(z) = \widetilde{H}_{i,K}(\widehat{F}_{Q_i}(z))$ and oracle AQR as $T_{i,*}^{\text{AQR}}(z) = H_i(F_{Q_i}(z))$; We add and subtract the intermediate term $H_i(\widehat{F}_{Q_i}(z))$:
\begin{align}
T_{i,K,n}^{\text{AQR}}(z) - T_{i,*}^{\text{AQR}}(z) &= \widetilde{H}_{i,K}(\widehat{F}_{Q_i}(z)) - H_i(F_{Q_i}(z)) \\
&= \widetilde{H}_{i,K}(\widehat{F}_{Q_i}(z)) - H_i(\widehat{F}_{Q_i}(z)) + H_i(\widehat{F}_{Q_i}(z)) - H_i(F_{Q_i}(z)) \\
&= (\widetilde{H}_{i,K} - H_i)(\widehat{F}_{Q_i}(z)) + H_i(\widehat{F}_{Q_i}(z)) - H_i(F_{Q_i}(z))
\end{align}
\end{proof}

% \textbf{Intuition}: We decompose the total approximation error into two sources: (1) errors in estimating the quantile function from finite source samples and finite quantile resolution, and (2) errors in estimating the target CDF from finite target samples.

\textbf{Intuition}: This decomposition separates two distinct error sources: (1) \emph{quantile estimation error} from approximating the true quantile function $H_i$ with the empirical piecewise-linear $\widetilde{H}_{i,K}$, and (2) \emph{CDF estimation error} from using the empirical target CDF $\widehat{F}_{Q_i}$ instead of the true $F_{Q_i}$.\\
The next lemma provides concentration bounds for empirical CDF estimation:
\begin{lemma}[Concentration via Dvoretzky-Kiefer-Wolfowitz Inequality]
\label{lem:dkw}
For any $\delta \in (0,1)$, with probability at least $1-\delta$:
\begin{align}
\|\widehat{F}_{Q_i} - F_{Q_i}\|_\infty &\leq \varepsilon_T(\delta, n_T) \\
\|\widehat{F}_{P_i} - F_{P_i}\|_\infty &\leq \varepsilon_S(\delta, n_S)
\end{align}
where $\varepsilon_\bullet(\delta, n) = \sqrt{\frac{1}{2n}\log\frac{2}{\delta}}$.
\end{lemma}

\textbf{Explanation}: The DKW inequality is a fundamental result in empirical process theory that provides uniform concentration bounds for empirical distribution functions. The bound decreases at the optimal parametric rate $O(n^{-1/2})$ with logarithmic dependence on the confidence level.\\
We next establish how CDF estimation errors transfer to quantile estimation errors:
\begin{lemma}[CDF-to-Quantile Error Transfer]
\label{lem:cdf-to-quantile}
Under the density bounds~\eqref{eq:density-bounds}:
$$\sup_{u \in [0,1]} |\widehat{F}_{P_i}^{-1}(u) - F_{P_i}^{-1}(u)| \leq \frac{\|\widehat{F}_{P_i} - F_{P_i}\|_\infty}{\underline{f}_i} \leq \frac{\varepsilon_S(\delta, n_S)}{\underline{f}_i}$$
\end{lemma}
\textbf{Key Insight}: The density lower bound $\underline{f}_i$ controls how CDF errors amplify into quantile errors. When the density is bounded away from zero, quantile estimation remains stable.

The following lemma bounds the discretization error from using finite quantiles:

\begin{lemma}[Finite-Quantile Discretization Error]
\label{lem:discretization}
Let $H_{i,K}$ be the piecewise-linear interpolant of the true quantile function $H_i$ on the uniform knots $\{u_j\}_{j=0}^K$. Then:
$$\|H_i - H_{i,K}\|_\infty \leq \frac{\|H_i''\|_\infty}{8} K^{-2} \leq \frac{L_i}{8\underline{f}_i^3} K^{-2}$$
\end{lemma}

\textbf{Explanation}: This is a classical result from approximation theory. Linear interpolation of a twice-differentiable function achieves quadratic convergence in the mesh size. The bound depends on the second derivative of the quantile function, which we control through our smoothness assumptions.

Finally, we bound how empirical estimation affects the interpolant:

\begin{lemma}[Knot Stability]
\label{lem:knot-stability}
The empirical and theoretical piecewise-linear interpolants satisfy:
$$\|\widetilde{H}_{i,K} - H_{i,K}\|_\infty \leq \max_{0 \leq j \leq K} |\widehat{q}_{i,j} - H_i(u_j)| \leq \frac{\varepsilon_S(\delta, n_S)}{\underline{f}_i}$$
\end{lemma}

\textbf{Reasoning}: Since both interpolants use the same piecewise-linear basis, their difference is controlled by the maximum knot error. The knot errors are bounded using Lemma~\ref{lem:cdf-to-quantile}.

\paragraph{Proof of Main Theorem.}

\textbf{Step 1: Uniform Error Bound.}
We combine all lemmas to bound the uniform error for each neuron. From Lemma~\ref{lem:error-decomp} and the triangle inequality:
\begin{align}
\sup_{z \in \mathbb{R}} |T_{i,K,n}^{\text{AQR}}(z) - T_{i,*}^{\text{AQR}}(z)| &\leq \|\widetilde{H}_{i,K} - H_i\|_\infty + \|H_i'\|_\infty \|\widehat{F}_{Q_i} - F_{Q_i}\|_\infty
\end{align}

We further decompose the quantile error using the triangle inequality:
$$\|\widetilde{H}_{i,K} - H_i\|_\infty \leq \|\widetilde{H}_{i,K} - H_{i,K}\|_\infty + \|H_{i,K} - H_i\|_\infty$$

Applying Lemmas~\ref{lem:discretization} and~\ref{lem:knot-stability}, along with $\|H_i'\|_\infty \leq 1/\underline{f}_i$ from~\eqref{eq:quantile-bounds}, we obtain with probability at least $1-2\delta$:
\begin{align}
\sup_{z \in \mathbb{R}} |T_{i,K,n}^{\text{AQR}}(z) - T_{i,*}^{\text{AQR}}(z)| &\leq \frac{L_i}{8\underline{f}_i^3} K^{-2} + \frac{\varepsilon_S(\delta, n_S)}{\underline{f}_i} + \frac{\varepsilon_T(\delta, n_T)}{\underline{f}_i} \label{eq:uniform-bound}
\end{align}

\textbf{Step 2: Per-Neuron MSE Bound.}
Since the oracle AQR achieves perfect recovery ($T_{i,*}^{\text{AQR}}(h_i^T) = h_i^S$), we have:
$\text{MSE}_i(T_{i,K,n}^{\text{AQR}}) = \mathbb{E}[(T_{i,K,n}^{\text{AQR}}(h_i^T) - h_i^S)^2] = \mathbb{E}[(T_{i,K,n}^{\text{AQR}}(h_i^T) - T_{i,*}^{\text{AQR}}(h_i^T))^2]$

This MSE is bounded by the squared uniform error:
$\text{MSE}_i(T_{i,K,n}^{\text{AQR}}) \leq \left(\sup_z |T_{i,K,n}^{\text{AQR}}(z) - T_{i,*}^{\text{AQR}}(z)|\right)^2$

Using the inequality $(a+b+c)^2 \leq 3(a^2 + b^2 + c^2)$ and~\eqref{eq:uniform-bound}, with probability at least $1-2\delta$:
\begin{align}
\text{MSE}_i(T_{i,K,n}^{\text{AQR}}) &\leq 3\left(\frac{L_i}{8\underline{f}_i^3}\right)^2 K^{-4} + \frac{3}{\underline{f}_i^2}\varepsilon_S(\delta, n_S)^2 + \frac{3}{\underline{f}_i^2}\varepsilon_T(\delta, n_T)^2 \label{eq:per-neuron-bound}
\end{align}

This completes the proof of the main theorem. \qed

\paragraph{Comparison with TTN and Practical Implications.}
Our finite-sample analysis reveals a fundamental difference between AQR and TTN:

\textbf{TTN Limitation}: From the main text analysis, TTN uses affine transformations $T_i^{\text{TTN}}(z) = \mu_i^S + \sigma_i^S \frac{z - \mu_i^T}{\sigma_i^T}$ that match first and second moments. However, $\text{MSE}(T^{\text{TTN}}) > 0$ whenever some corruption $k_i$ is non-affine, representing a \emph{constant bias} that does not decrease with sample size or computational resources.

\textbf{AQR Advantage}: Our Theorem~\ref{thm:finite-sample-aqr} shows that for each neuron, AQR's error decreases as:
\begin{itemize}
    \item $O(K^{-4})$ with quantile resolution (quartic convergence from linear interpolation)
    \item $O(n_S^{-1/2})$ and $O(n_T^{-1/2})$ with sample sizes (optimal parametric rates)
    \item Additional logarithmic factors $\log(1/\delta)$ accounting for confidence level
\end{itemize}

\textbf{Practical Consequence}: For sufficiently large values of $K$, $n_S$, and $n_T$, we have $\text{MSE}_i(T_{i,K,n}^{\text{AQR}}) < \text{MSE}_i(T_i^{\text{TTN}})$ for each neuron, providing theoretical justification for AQR's superior empirical performance in domain adaptation tasks.\\

\textbf{Extension to Full Network}: The total network MSE is simply $\text{MSE}(T_{K,n}^{\text{AQR}}) = \sum_{i=1}^m \text{MSE}_i(T_{i,K,n}^{\text{AQR}})$. Since each neuron achieves the convergence rates established in our theorem, the network-level performance inherits the same asymptotic behavior, making AQR superior to TTN at the network level as well.

\section{Tail Calibration Strategy: Gaussian Estimation}
\label{app:gaussian-tail}
Here we present additional details on the \emph{Gaussian Estimation} strategy described in \ref{sec:tails}. Rather than using empirical extrema, this approach models both source and target distributions as Gaussian, and estimates their theoretical tail values, as shown in Equation \ref{eq:gaussian_estimation}.
\begin{equation}
\text{AQR}(x) = 
\begin{cases} 
\left( \dfrac{x - Q(0)^T}{Q(1)^T - Q(0)^T} \cdot \left( Q(1)^S - Q(0)^S \right) \right) + Q(0)^S & x < p_1^T \\
\left( \dfrac{x - Q(99)^T}{Q(100)^T - Q(99)^T} \cdot \left(Q(100)^S - Q(99)^S \right) \right) + p_{99}^S & x \geq p_{99}^T 
\end{cases}
\label{eq:gaussian_estimation}
\end{equation}

Here the source and target quantile functions are
\[
Q(p)^S = \beta + \gamma\,\Phi^{-1}(p), \qquad
Q(p)^T = \mu(X) + \sigma(X)\,\Phi^{-1}(p),
\]
where \(\Phi^{-1}(p)=\sqrt{2}\,\mathrm{erf}^{-1}(2p-1)\) is the probit transform.  
\(X\) denotes all activations from a given channel of a layer for the test input.
\end{document}